\documentclass[trsc,nonblindrev]{informs3} 

\OneAndAHalfSpacedXI 

\usepackage{natbib}
 \bibpunct[, ]{(}{)}{,}{a}{}{,}%
 \def\bibfont{\small}%
\TheoremsNumberedThrough     

\EquationsNumberedThrough    

\usepackage{subcaption}
\usepackage[font=normalsize]{subcaption}

\usepackage{multirow,tabularx}
\usepackage{booktabs}
\usepackage{multirow}

\usepackage{stmaryrd}

\newenvironment{proof}{Proof:}{\hfill$\square$}

\newcommand\oprocendsymbol{\hbox{$\Diamond$}}
\newcommand\oprocend{\relax\ifmmode\else\unskip\hfill\fi\oprocendsymbol}

\usepackage[hidelinks]{hyperref}
\usepackage{varioref}
\usepackage{cleveref}

\begin{document}


\RUNAUTHOR{Yue and Smith}

\RUNTITLE{Robotic-Based Compact Storage and Retrieval Systems}

\TITLE{\Large Minimizing Robot Digging Times to Retrieve Bins in Robotic-Based Compact Storage and Retrieval Systems}

\ARTICLEAUTHORS{%
\AUTHOR{Anni Yue and Stephen L. Smith}
\AFF{Department of Electrical and Computer Engineering, University of Waterloo,  Waterloo, ON, Canada \\  \EMAIL{anni.yue@uwaterloo.ca; stephen.smith@uwaterloo.ca}}
} 

\ABSTRACT{Robotic-based compact storage and retrieval systems provide high-density storage in distribution center and warehouse applications. 
In the system, items are stored in bins, and the bins are organized inside a three-dimensional grid. Robots move on top of the grid to retrieve and deliver bins. 
To retrieve a bin, a robot removes all bins above one by one with its gripper, called bin digging. 
The closer the target bin is to the top of the grid, the less digging is required to retrieve the bin.
In this paper, we propose a policy to optimally arrange the bins in the grid while processing bin requests so that the most frequently accessed bins remain near the top of the grid. This improves the performance of the system and makes it responsive to changes in bin demand. 
Our solution approach identifies the optimal bin arrangement in the storage facility, initiates a transition to this optimal set-up, and subsequently ensures the ongoing maintenance of this arrangement for optimal performance. 
We perform extensive simulations on a custom-built discrete event model of the system. 
Our simulation results show that under the proposed policy more than half of the bins requested are located on top of the grid, reducing bin digging compared to existing policies.
Compared to existing approaches, the proposed policy reduces the retrieval time of the requested bins by over 30\% and the number of bin requests that exceed certain time thresholds by nearly 50\%. 
}



\maketitle

%
\section{Introduction}
\label{sec:Introduction}
Warehousing and material handling systems are currently adapting to changes in customer purchasing behavior and modern distribution strategies. 
Online shopping, same- or next-day delivery, store pick-up, metro e-commerce centers (MECs), micro-fulfillment centers (MFCs), and dark stores require warehouses to have high efficiency, high reliability, high storage density, high energy efficiency, and low cost, which overlap with the well-known objectives of automated storage and retrieval systems (AS/RS).
Among various types of AS/RS, robot-based compact storage and retrieval systems (RCS/RS) stand out as a great solution for storing small or miniature items in limited space, providing high throughput, and maintaining order-picking accuracy.

\begin{figure}[ht]
    \begin{subfigure}{0.49\linewidth}
         \centering
         \includegraphics[width=\textwidth]{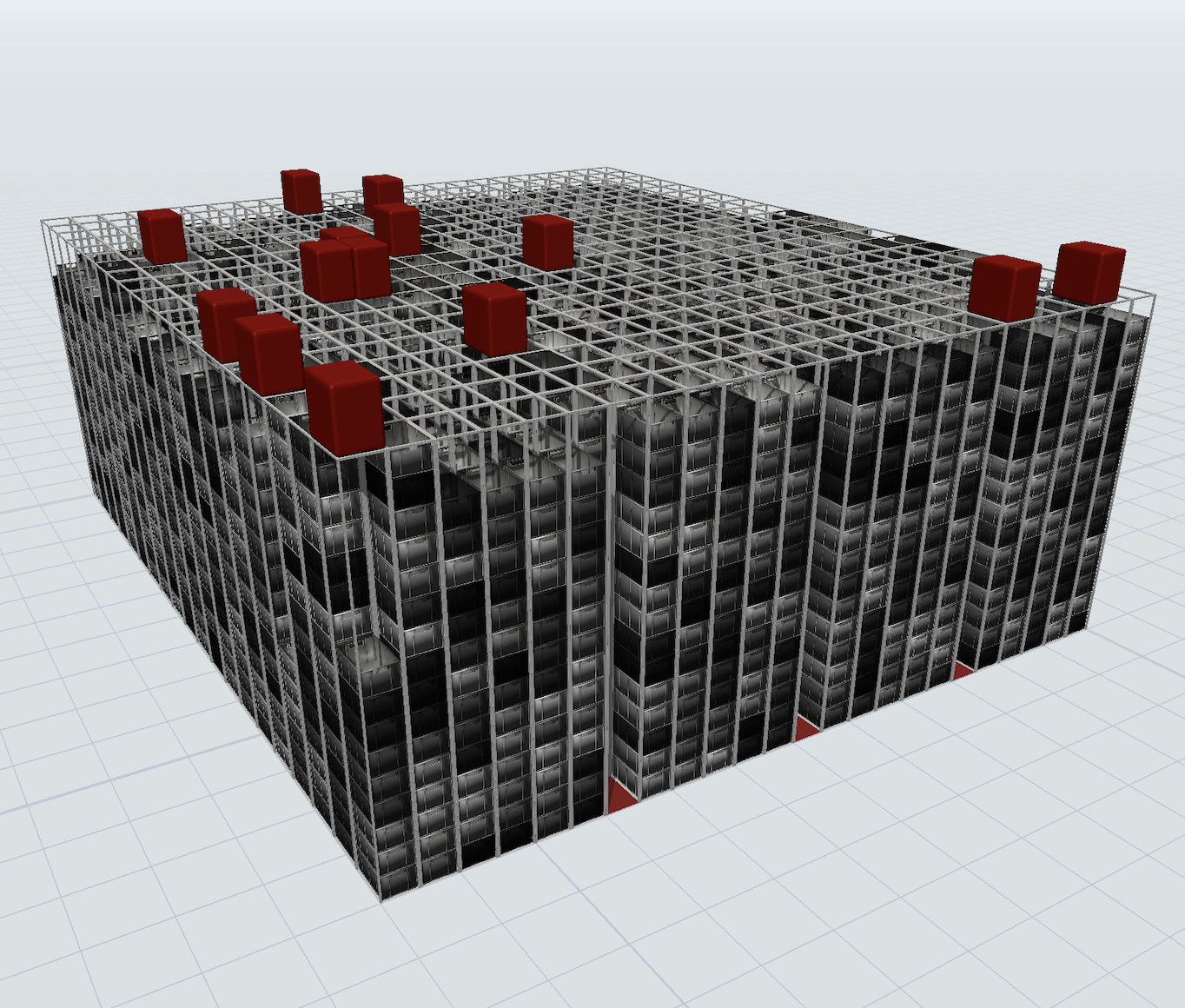}
         \caption{Oblique view}
         \label{fig:obview}
     \end{subfigure}
     \hfill
     \begin{subfigure}{0.49\linewidth}
         \centering
         \includegraphics[width=\textwidth]{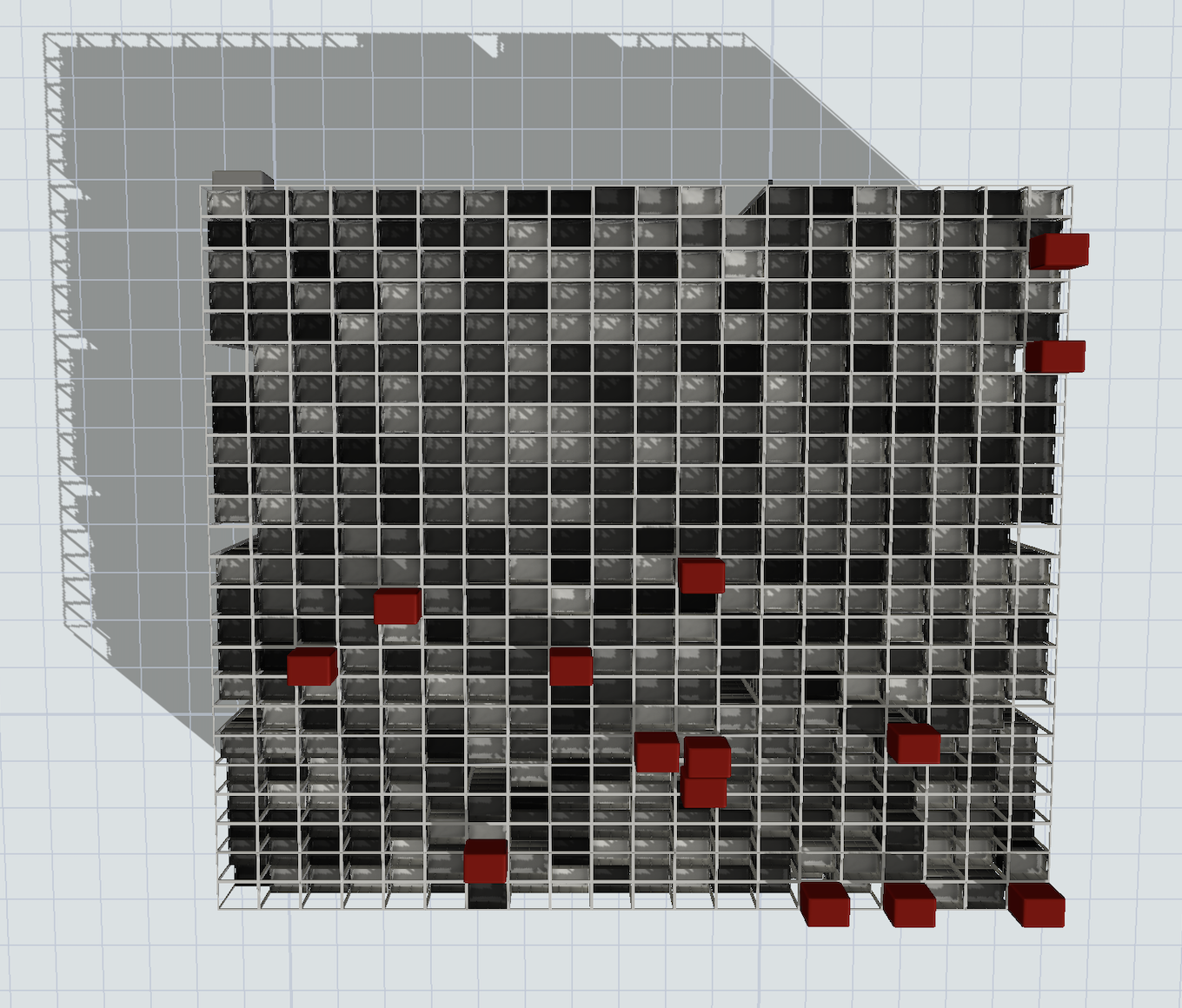}
         \caption{Top view}
         \label{fig:tpview}
     \end{subfigure}
    \caption{Robotic-Based Compact Storage and Retrieval System}
    \label{fig:RCSRS}
\end{figure}

RCS/RS, such as the systems of AutoStore and Ocado, store small items in standardized bins that are organized in a three-dimensional grid (see \Cref{fig:RCSRS}). 
Robots work on top of the grid to retrieve the bin containing the requested item and deliver the bin to the workstation (goods-to-person or GTP), where a worker (human or machine) completes the order-picking or storage-replenishing tasks. 
The bins returned from workstations are placed on top of the grid, which makes bins with high-demand items move to the top of the grid and bins with low-demand items sink to the bottom.
\citet{autostoreCasesWeb} has successfully deployed the system for companies including
PUMA and Best Buy in the United States, MS Direct in Switzerland, Elotec and XXL in Norway, Siemens in Germany, and Decathlon in Canada.
\citet{galka2020autostore} publish an online survey of 64 participants (AutoStore users) that provides the system design parameters, discusses the advantages and disadvantages of the system, describes the factors that need to be considered in each phase of the system (planning, realization and operation), etc.

Existing research on RCS/RS has evaluated storage policies with different types of storage stack and reshuffling methods for a given inventory and system configuration \citep{zou2018operating}, and has investigated policies to add new bins to existing systems \citep{beckschafer2017simulating}.
However, a key challenge that remains unaddressed is how the system should respond when there is a sudden change in the demand for bins (i.e., how often each bin is requested). Such changes can occur due to seasonal turnover of products or promotional activities and sales.
Our objective is to ensure high throughput of the system without suspending order picking.
To achieve this goal, we study the optimal arrangement of bins in the grid and propose methods that incrementally transition an initial arrangement to an optimal one while processing bin requests.  

\subsection{Contributions}
The key contributions of this paper are three-fold:
\begin{enumerate}
    \item Given the demand level of each bin, we derive the optimal bin arrangement within the grid, which guarantees the minimum expected time to retrieve one bin from the system.
    \item On the basis of the optimal bin arrangement, we build a set of target bin arrangements within the grid, which reduces the time to retrieve a sequence of bins from the system (with returns).
    \item We propose a policy that selects storage locations in the grid for bins returned from workstations, which gradually transforms any given initial bin arrangement into one that is in the set of target bin arrangements in the grid.
\end{enumerate}

We build discrete event simulation models to validate the proposed approach and compare it with two existing methods \citep{zou2018operating}.
The simulation results show that our approach outperforms their methods in terms of less digging work, less time to retrieve the requested bins, and less work time for the robots to complete the same amount of bin requests.

\subsection{Literature Review} 
\textbf{History of AS/RS:} 
Over the past few decades, with the rapid evolution of technology, the material handling industry has continued to make AS/RS more efficient, intelligent, and cost-effective.
\citet{roodbergen2009survey} and \citet{gagliardi2012models} conduct surveys covering the developments and characteristics of various AS/RS.
\citet{olsson2019framework} review the rapid growth of last-mile logistics and conclude directions for further research in the field of material handling.
Meanwhile, a variety of optimization methods have been developed for order-picking planning to deal with different types of AS/RS \citep{van2018designing}.

AS/RS were originally designed to handle bulky cargoes \citep{gharehgozli2015scheduling}. 
Today, the system use-cases have been expanded, allowing them to effectively manage storage units of different sizes at different stages of the logistics process.
\citet{azadeh2019robotized} review the development of modern AS/RS and emphasize the trend of using robotic technology in distribution centers. 
Their work highlights autonomous vehicle-based storage and retrieval systems (AVS/RS, which are aisle-based shuttle systems), robotic-based compact storage and retrieval systems (RCS/RS, which are grid-based shuttle systems), and robotic mobile fulfillment systems (RMFS) in new robotized warehouses, and identifies gaps for further research.

In contrast to AVS/RS and RCS/RS, which use robots to deliver containers to workstations, RMFS uses automated guided vehicles (AGVs) or autonomous mobile robots (AMRs) to move racks to workstations.
\citet{bozer2018simulation} compare the performance of AVS/RS and RMFS in simulation and discuss their benefits and limitations.
Kiva Systems is a typical RMFS, which is used by clients such as online retailers, where each order contains only a few items with internal connections (e.g., usually ordered together). 
The system can store such items on the same rack, and thus one delivery can retrieve multiple required items, improving the picking speed. 
A considerable amount of work has been done to optimize the assignment of robots to workstations~\citep{zou2017assignment}, storage racks~\citep{weidinger2018storage}, and storage zones~\citep{roy2019robotf} within the facility.
To improve system performance, models are built to suggest rack layout in the pre-design stage~\citep{wang2020travel} and to assign and sequence orders and racks to workstations~\citep{valle2021order}. 
When investigating RCS/RS, we can refer to the above work, since those assignments also happen in RCS/RS.
However, RMFS do not fully utilize the vertical space (limited by the height of the rack) and the horizontal space (reduced by floor space for robot movements) in a warehouse, resulting in a lower storage density compared to RCS/RS.

RCS/RS share similarities with AVS/RS, and they face similar challenges.
Both AVS/RS and RCS/RS use robots to retrieve and move containers (e.g., bins).
The horizontal motions of the robots in both systems are very similar. The robots travel in x and y directions on fixed paths or tracks. 
Generally, there are multiple robots that work collaboratively in a system to improve throughput. Therefore, both systems require dynamic routing, congestion management, and collision avoidance to improve their reliability and safety within compact and densely packed structures.
The two systems differ in the vertical movements of the robots in the z direction.
AVS/RS store bins on shelves and use lifts to move robots between layers, thus robots are able to move on any layers.
RCS/RS remove shelves and stack bins on top of each other. To retrieve a bin, a robot needs to use its gripper to remove the bins above the target bin one by one.
Due to the big difference, some of the research findings and conclusions about AVS/RS are inappropriate or inapplicable to RCS/RS.
In the following, we first present extensive existing studies on AVS/RS.
Then, we present the studies on RCS/RS that we are able to find.

\textbf{AVS/RS:} 
\citet{malmborg2002conceptualizing} propose Markov chain models to analyze system performance.
They carefully explain the design parameters of the system (e.g., the number of storage columns, tiers, lifts and vehicles) and state the evaluators to estimate the performance of the system (e.g., order picking cycle time and vehicle utilization).
Subsequent studies have furthered our understanding of AVS/RS by providing more detailed and comprehensive analyses.
In the following, we provide a concise overview from two key perspectives: design factors and performance evaluators. 

To design an AVS/RS that best fits the operation requirements, a designer should consider design parameters such as
storage policies \citep{malmborg2002conceptualizing}, 
scheduling rule, input/output (I/O) locations, interleaving rule \citep{ekren2010simulation2},
rack configuration \citep{ekren2010simulation1,marchet2013development}, 
dwell point policies, location of cross-aisles \citep{roy2015queuing},
number of tires, type of shuttles \citep{tappia2017modeling},
and system types, i.e., vertical and horizontal \citep{azadeh2019design}.

To effectively assess the performance of an AVS/RS for various purposes, there are a range of performance evaluators to choose from, including
order picking cycle time, vehicle utilization \citep{malmborg2002conceptualizing}, 
storage and retrieval cycle times, system utilization, throughput capacity \citep{malmborg2003interleaving},
transaction service times, transaction waiting time \citep{kuo2007design}, 
resource utilization, costs and space requirements \citep{fukunari2009network}, 
transaction cycle time \citep{marchet2012analytical}, 
throughput of each station, mean number of jobs at each station, mean residence time, and throughput of the entire network \citep{ekren2012approximate}. 
 
\textbf{RCS/RS:}
The first AutoStore system appeared in the early $21$st century \citep{firstautostoreweb}. 
However, there is limited literature directly related to the system. Most research focused on RCS/RS has been published in the past five years.
We see that RCS/RS has garnered increasing attention recently and is one of the main research directions in the future \citep{azadeh2019robotized}.

\citet{zou2018operating} build analytical models using reduced semi-open queuing networks (SOQN) to predict system throughput time, expected waiting time of orders for robots and robots for workstations, and utilization of robots and pickers. 
They validate the models through simulations, where they compare two storage policies (dedicated and shared), two storage stacks (random and zoned), and two reshuffling methods (immediate and delayed).
They use real case data to verify analytic models for the shared storage policy with the random storage stack and immediate reshuffling. 
They also suggest the optimal length-to-width dimension for each strategy.   
However, their models have a number of simplifications (e.g., when moving on top of the grid, robots go along with the shortest path without stopping and can cross obstacles), and they do not consider the new bins being added from outside the grid.
\citet{beckschafer2017simulating} propose two input storage policies to add new storage items to the system. They examine their policies in a discrete event simulation environment, where robots use the A$^*$ algorithm \citep{hart1968formal} to plan the route for robots moving on top of the grid.

The performance of RCS/RS, similar to that of AVS/RS, is based on various design parameters. These parameters include the dimension of the grid, the number of robots and workstations, the order structure, and the storage policy. 
Due to the complex interactions between these factors, the impact of each individual factor needs to be accurately analyzed through simulation studies.
Using simulation studies, \citet{simulationbasedanalysis} examine the impact of the number of robots in the system while keeping other factors constant. They find that the number of robots should align with the size of the grid and the number of workstations, as redundant robots can hinder system performance. Additionally, they emphasize that the distribution of bin requests plays a crucial role in determining the suitability of the system.
\citet{trost2022simulation} develop a discrete event simulation model to analyze the impact of many design parameters on system performance. Their findings reveal that
1) higher or fuller stacks result in lower throughput,
2) the location of the requested bin within a stack significantly affects throughput,
and 3) the optimal number of robots for a system is determined with great precision by means of simulations.

Some studies explore the broader aspects of RCS/RS operations (i.e., upstream and downstream assignments), including order rearrangement before entering the system, order picking at workstations, and order packing after leaving the system.
\citet{tjeerdsma2019redesign} redesign the AutoStore order processing line to improve productivity, i.e., the number of orders per man hour. 
The proposed improvements include the balance of workload between offline and online workstations, the minimization of order travel distances, the implementation of a cellular layout for similar process procedures, and the integration of automation into the packing process.  
\citet{ko2022rollout} propose a rollout heuristic algorithm to determine the optimal order processing sequence, effectively minimizing the total number of processed bins.
In reviewing synchronization difficulties in warehouses with parts-to-picker, \citet{boysen2022review} conclude that selecting an appropriate workstation configuration can significantly improve throughput efficiency.
At the workstation, robot arms replace humans to improve picking efficiency and accuracy. \citet{yang2022sesr} propose an AI-picking solution designed for items stored in small bins. 

Both the studies focused on the system itself and the broader studies related to the system are providing us with an increasingly complete and clear picture of RCS/RS.
However, a situation that requires further investigation is how the system should respond to changes in the frequency of requests for bins.  In this paper, we propose a solution for this situation that ensures high throughput and small retrieval times of bins but does not interrupt normal order-picking.

\subsection{Organization}
\Cref{lab: Sec system description} presents a detailed description of RCS/RS.
In \Cref{lab: Sec problem formulation}, we formalize the system, explain the notation used in the paper, and define the bin grid configuration problems.
In \Cref{lab: Sec BGC} and \Cref{lab: Sec Policy}, we derive an optimal bin grid configuration and propose a layer complete policy to minimize the bin retrieval time. 
In \Cref{lab: Sec result}, we introduce our simulation environment, evaluate the performance of the proposed policy, and compare our solution with two existing policies.
In the end, we draw conclusions and discuss the possibilities of future research in \Cref{lab: Sec conclusion}.

\section{System Description}
\label{lab: Sec system description}
\Cref{lab: the system} introduces key components of the specific RCS/RS studied in this paper.
\Cref{lab: the process} explains the basic operations and the workflow of the system.
\subsection{Robotic-Based Compact Storage and Retrieval System}\label{lab: the system}

The robotic-based compact storage and retrieval system provides a superior level of storage density by eliminating the wasted space of traditional shelving. In the rest of the paper, we use \emph{the system} to represent the RCS/RS. 

\begin{figure}[ht]
    \centering
    \begin{subfigure}{0.49\textwidth}
    \centering
    \includegraphics[width=\linewidth]{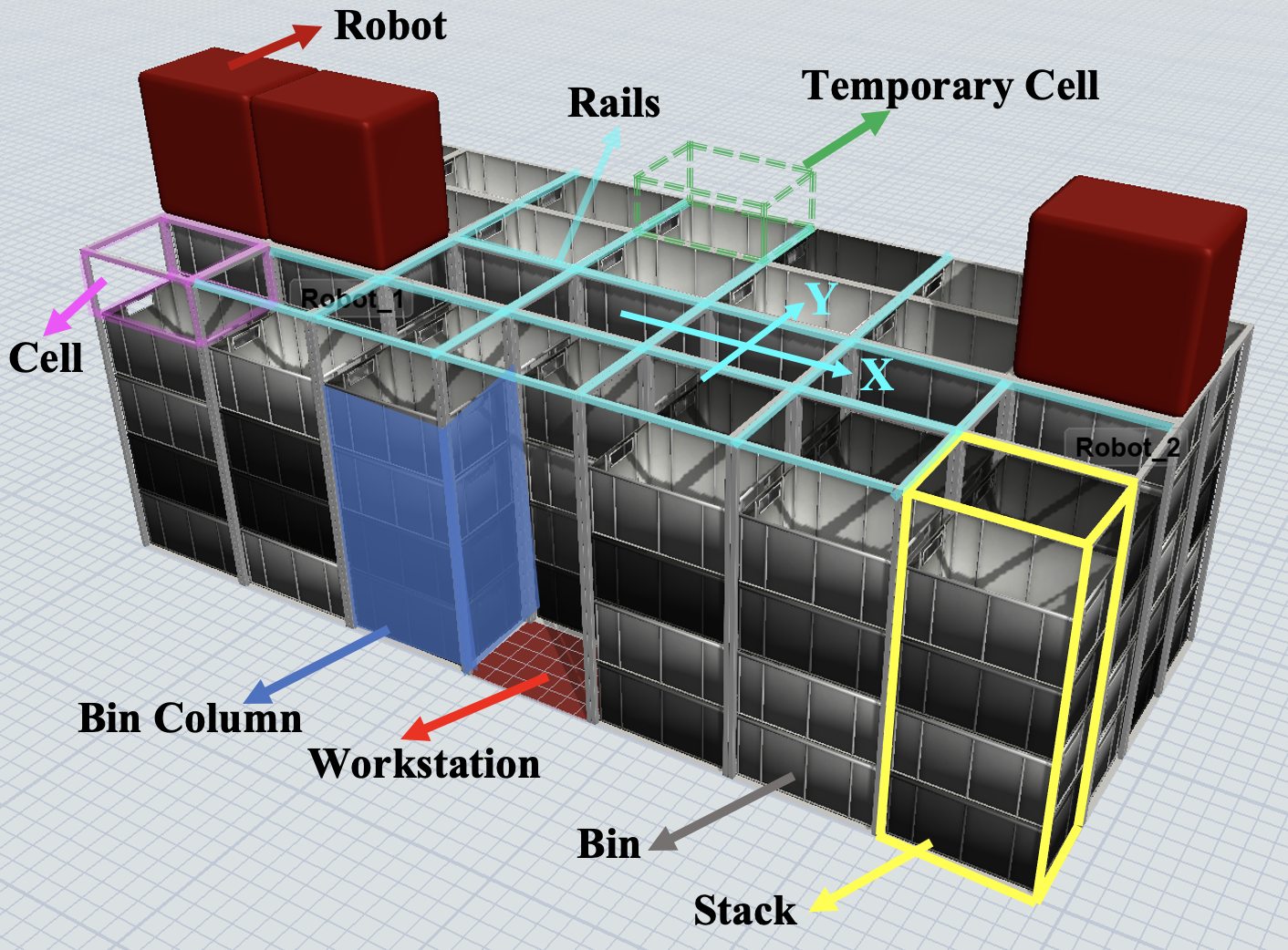}
    \caption{RCS/RS Key Components}
    \label{fig:SystemComponents}
    \end{subfigure}\hfill
    \begin{subfigure}{0.49\textwidth}
    \centering
    \includegraphics[width=0.95\textwidth]{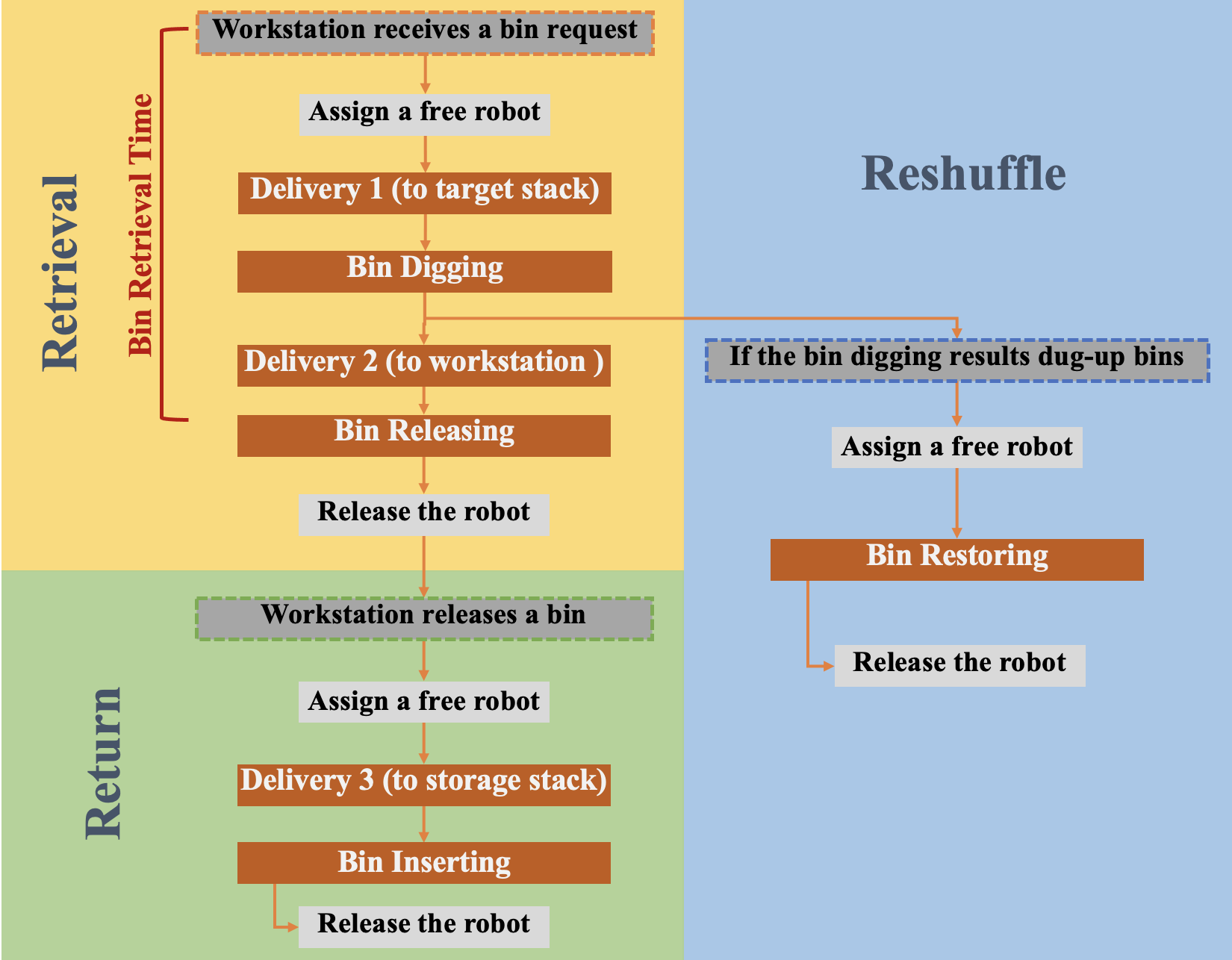}
    \caption{Flowchart of The System}
    \label{fig:my_flow}
    \end{subfigure}
    \caption{Introduction of RCS/RS}
\end{figure}

\Cref{fig:SystemComponents} presents the essential components of the system.
All small inventory items are stored in standardized \emph{bins}. 
The bins are stacked from bottom to top to form a \emph{bin column}, and each bin column is organized in a \emph{stack}.
The \emph{stack} is a modular vertical frame that maintains only the vertical alignment of the bins, but does not support each bin like a shelf.
Each bin space in the stack is a \emph{cell} (normal), and there is a \emph{temporary cell} above each stack. 
Depending on the shape of the warehouse footprint, the stacks are adjacently aligned to form a \emph{grid} with a flat top. 

The grid holds all the components together.
At the top of the grid, \textit{rails} are constructed in X and Y directions and allow robots to travel to any stack (we call this robot motion the \emph{delivery motion} in the rest of the paper).
The robot accesses a target bin from the top of the stack by removing, if there is any, the bins above the target bin one by one using its gripper (we call this robot motion the \emph{gripper motion} in the rest of the paper).
The dug-up bins are placed on top of nearby stacks that have empty cells (normal and temporary).
With the gripper motion and the delivery motion, any robot can access any bin in the grid and deliver the bin to the \emph{workstation}. 
At the workstation, a worker or robot processes the fulfillment and/or replenishment work.

In this paper, we study only the system and do not consider the impact of the warehouse management system (WMS). 
While a WMS enhances the system performance by reordering bin requests before they enter the system, it is important to note that the impact of this optimization varies from case to case. Different WMS use different logic, leading to varying effects on system performance.
To isolate this impact, we have the following assumption:
\begin{assumption}\label{lab: asp bin request}
    The system processes incoming bin requests on a first-come-first-served basis.
\end{assumption}

\subsection{The Working Process}
\label{lab: the process}
A free robot with no assigned task waits at the top of the stack or workstation where it finished its last task (\emph{dwelling point}) and enters a FIFO queue immediately after finishing its last task.
Each time a workstation receives a \emph{bin request}, the system processes the bin request following the workflow presented in \Cref{fig:my_flow}. 
Below is a detailed description of each task:
\begin{description}
    \item[\textbf{Delivery 1 (to target stack):}] A free robot moves from its dwelling point to the \emph{target stack} that stores the \emph{target bin}.
    \item[\textbf{Bin Digging:}] The same robot (in Delivery 1) digs up bins above the target bin one by one, if there are any, and places the dug-up bins on top of nearby stacks (those stacks and dug-up bins are then \emph{blocked}). The robot retrieves the target bin and brings it to the top of the target stack.
    \item[\textbf{Delivery 2 (to workstation):}] The same robot (in Delivery 1) delivers the target bin to the assigned workstation via the shortest route. 
    \item[\textbf{Bin Releasing:}] The same robot (in Delivery 1) drops the target bin at the workstation.
    \item[\textbf{Bin Restoring:}] If the bin request induces the dug-up bins in Bin Digging, a free robot leaves its dwelling point and restores the previous dug-up bins to the target stack in the same order. (After restoring the previous dug-up bins from the nearby stacks, these blocked bins and stacks are \emph{freed}.)
    \item[\textbf{Delivery 3 (to storage stack):}] After finishing the work with the target bin at the workstation, a free robot moves from its dwelling point to the workstation, collects the target bin, and delivers the bin to the storage stack of the grid.
    \item[\textbf{Bin Inserting:}] The same robot (in Delivery 3) drops the target bin on top of the storage stack.
\end{description}
In this mechanism, bins with high demand move to the top of the grid, while bins with low demand sink to the bottom of the grid. 
\section{Problem Formulation}
\label{lab: Sec problem formulation}
In this section, we first explain and list the essential notations to facilitate the system in \Cref{lab:Environment Model}.
Then, we introduce the bin retrieval cost and the expected bin retrieval cost in \Cref{sec:Computation of The Costs}.
Ultimately, we define the optimal bin grid configuration problem and the bin rearrangement problem in \Cref{lab:BGC problem}. 
\subsection{Storage System Model}
\label{lab:Environment Model}
To improve the performance of the system, we focus on the gripper motion and model the system as follows. 
Appendix~\ref{APX:Notations for System Analysis} summarizes essential notations of the model
and \Cref{fig:Notation1} presents an example to help the reader understand the notations.
In \Cref{Implementation Details}, we have additional design parameters to build a simulation model that mimics reality.
In the paper, given $X\in\mathbb{N}^+$, we use $\llbracket X\rrbracket$ to denote the set of integers ranging from $1$ to $X$, i.e., $\llbracket X\rrbracket=\{1,2,\dots,X\}$.

\begin{figure}[ht]
    \centering
    \begin{subfigure}[t]{0.48\textwidth}
    \centering
    \includegraphics[width=\textwidth]{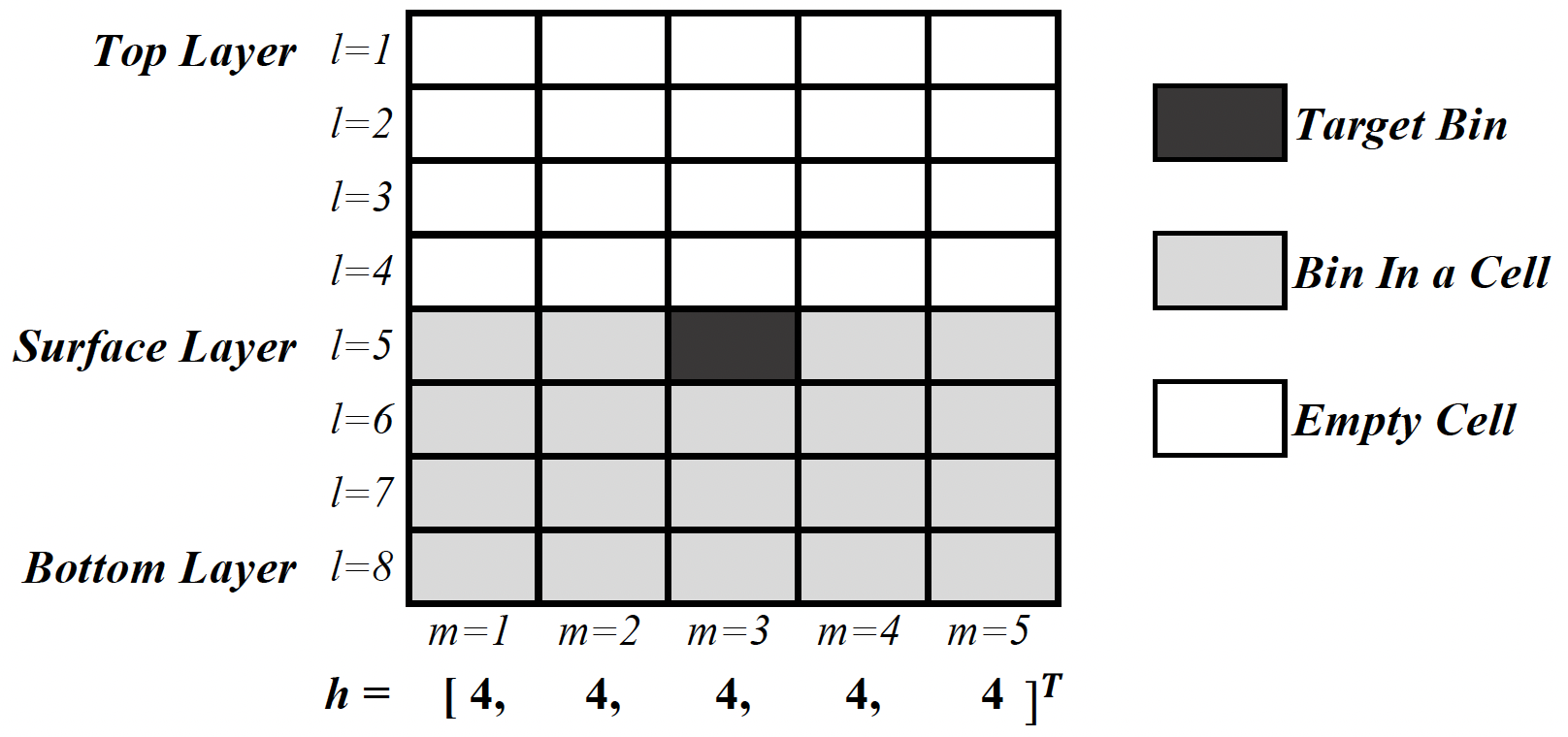}
    \caption{A system: $M=5$, $H=8$, $N=20$, $storage\ capacity = 40$, $h_c=4$ and $h_e=4$.}
    \label{fig:Notation1}
    \end{subfigure}\hfill
    \begin{subfigure}[t]{0.48\textwidth}
    \centering
    \includegraphics[width=0.85\textwidth]{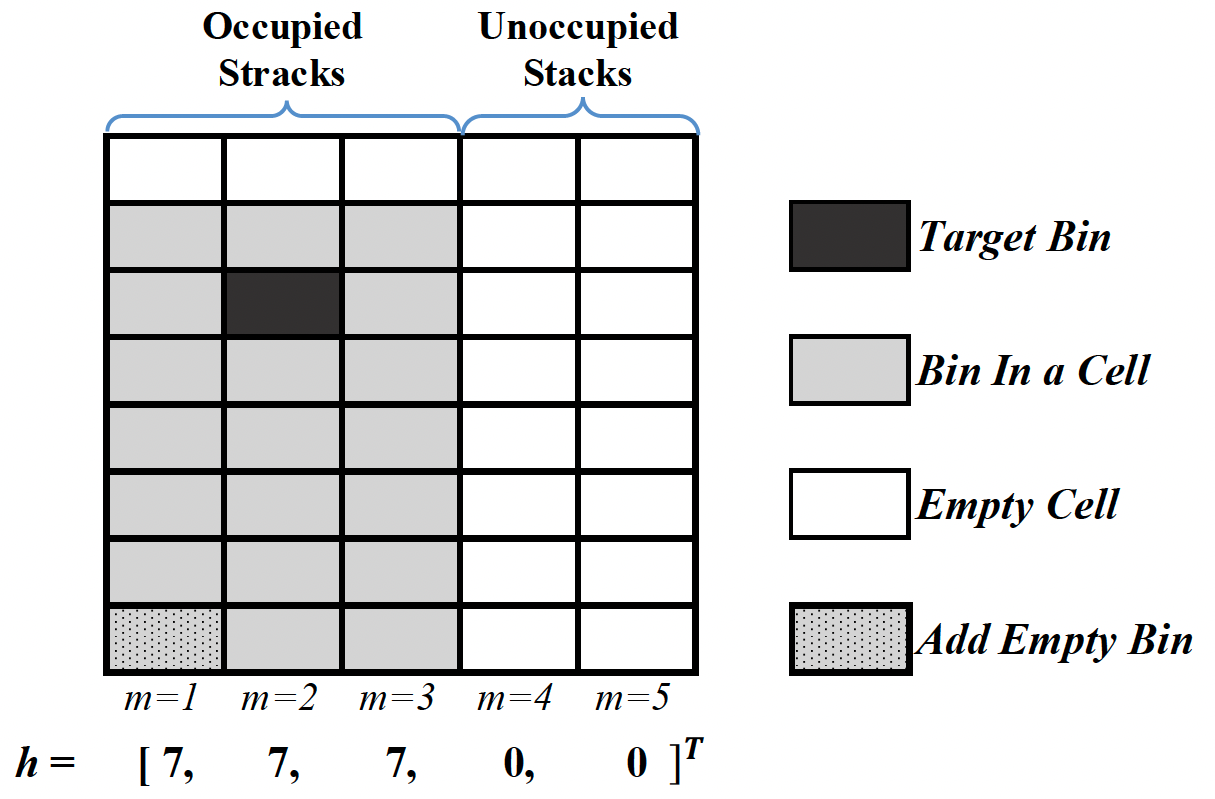}
    \caption{A system has fill level $h_c=7$, occupies three stacks ($m_f=3$), and has one empty bin.}
    \label{fig:Notation2}
    \end{subfigure}
    \caption{Example illustration of notations of the RCS/RS.}
\end{figure}

Consider a grid consisting of $M$ stacks, and $m\in \llbracket M\rrbracket $ represents the stack ID.
Each stack offers $H$ cells (i.e., a stack can store up to $H$ bins), and we define $H$ as the \emph{stack height}.
We use the term \emph{layer}, $l\in \llbracket H\rrbracket $, to locate a cell in a stack, where the top cell of a stack is in layer $1$ (\emph{top layer}) and the bottom cell of a stack is in layer $H$ (\emph{bottom layer}).
The \emph{surface layer} is the layer where the topmost bin is located.
The \emph{storage capacity} of the system, $M\cdot H$, is the total number of cells available for bin storage in the grid.

Within storage capacity, the grid stores $N$ bins (i.e., $N\leq M\cdot H$), and $n\in \llbracket N\rrbracket $ represents the \emph{bin ID}.
The \emph{popularity} of bin $n$, $0\leq p_n\leq 1$, indicates the frequency of bin $n$ being requested. 
Without losing generality, we normalize the bin popularity and ensure $\sum_{n\in \llbracket N\rrbracket }p_n=1$.
Note that we assign the bin ID according to the bin popularity, where the bin with a higher bin popularity is assigned with a smaller bin ID, i.e., $\forall i,j\in\llbracket N\rrbracket $ and $i<j$, we have $p_i\geq p_j$.

We use the term \emph{fill level}, $h_m$, to represent the number of bins stored in stack $m$, where $0\leq h_m\leq H$. 
The system has a \emph{fill level vector} $\textbf{\emph{h}}$, including fill levels of all stacks in the grid, which is defined as $\textbf{\emph{h}}=[h_1,h_2,\dots,h_M]^\top.$

\citet{zou2018operating} considered the percentage ($\tau$) of the total cells that were put aside for potential expansion in the future, and set the number of cells for bin storage in each stack as an integer $\overline{H}=H(1-\tau)$. Under their shared storage policy coupled with random storage stacks, $\textbf{\emph{h}}=[\overline{H},\overline{H},\dots,\overline{H}]^\top$ is the only possible value for a system (see \Cref{fig:Notation1}).

In our system, we remove the limitation of using all stacks and set the fill level to an integer $h_c$ for occupied stacks, where $\overline{H}\leq h_c\leq H$. 
We use the term \emph{empty level}, $h_e$, to represent the number of empty cells in occupied stacks, where $h_e=H-h_c$.
When processing bin requests, the initially occupied stacks need to maintain the fill level.
For the initially unoccupied stacks, they remain empty all the time.
Since our objective is to improve the performance of the system by reducing digging work, and here we consider only the gripper motion in our analysis, the positions of the occupied stacks in the grid do not affect the analysis. Therefore, we have the following remark:
\begin{remark}[Uniformly Occupied Stacks]\label{rmk:occupied stacks}
    (See \Cref{fig:Notation2}.)
    The number of occupied stacks is $m_f=\lceil N/h_c\rceil$.
    For the case when $N$ is not devisable by $h_c$, we add $m_f\cdot h_c-N$ empty bins. 
    We assign the newly added empty bins the bin ID as $N+1,\dots, m_f\cdot h_c$ and set their popularity at $0$. Finally, we update $N$ to $m_f\cdot h_c$.
    Consequently, $\textbf{\emph{h}}=[h_1,\dots,h_{m_f},h_{m_f+1},\dots,h_{M}]^\top=[h_c,\dots,h_c,0,\dots,0]^\top$. 
\end{remark}

The uniformly occupied stacks structure reduces the time to dig up some bins, but may need to dig up more bins above the target bin, which creates a trade-off between the average digging time and the number of bins that need to be dug out.

To describe the arrangement of the bins in the grid, we introduce the following definitions:
\begin{definition}[Bin Grid Configuration, BGC]
A bin grid configuration is an $H \times M$ matrix $\mathbf{B}$ where each entry $\mathbf{B}_{l,m}$ gives the ID of the bin in stack $m$ and layer $l$ within the grid. If a cell is empty, then $\mathbf{B}_{l,m}$ is set to $0$.
\end{definition}

We use $p_{\mathbf{B}_{l,m}}$ to represent the popularity of the bin stored in the cell of stack $m$ and layer $l$, where $p_{0}=0$. 
With all the bins in the grid, $\sum_{l\in \llbracket H\rrbracket }\sum_{m\in \llbracket M\rrbracket }p_{\mathbf{B}_{l,m}}=1$.

\begin{definition}[The Set of All Feasible BGCs]
Let $\Phi\subset N^{H \times M}$ be the set of all feasible bin grid configurations. This set contains all matrices $ \mathbf{B}$ of the following form:
\begin{equation}
        \begin{aligned}
            \mathbf{B} &=  \begin{bmatrix}
                            \mathbf{0}_{h_e\times m_f}    & \mathbf{0}_{h_e\times (M-m_f)}\\
                            \mathbf{B}^{'}_{h_c\times m_f}       & \mathbf{0}_{h_c\times (M-m_f)}
                        \end{bmatrix}\textrm{,}
        \end{aligned}
    \end{equation}
    where $m_f = \lceil N/h_C\rceil$ is the number of occupied stacks and $h_e+h_c=H$.
    Each nonzero entry of $\mathbf{B}^{'}$ is a unique value of $\llbracket N\rrbracket $.
\end{definition}

\begin{definition}[The Equivalent Class of a BGC]
\label{rmk: eqc_r}
    In the set of all feasible BGCs $\Phi$, we define the equivalence relation $\sim$ as follows:
    given $\overline{\mathbf{B}}\in\Phi$ and $\mathbf{B}\in\Phi$, then $\overline{\mathbf{B}}\sim\mathbf{B}$ if and only if for each column $i\in\{1,\ldots,m_f\}$ the entries in $\overline{\mathbf{B}}$ and $\mathbf{B}$ are the same (with possibly different ordering).
    We let $[\overline{\mathbf{B}}]:=\{\mathbf{B}\in\Phi:\overline{\mathbf{B}}\sim\mathbf{B}\}$ denote the equivalence class to which $\overline{\mathbf{B}}$ belongs.
\end{definition}

With all the definitions and notation introduced in this section, we can dive into the system performance analysis. 
Next, we set the evaluation variables and define the problems studied in this paper. 

\subsection{Costs Associated with Bin Retrieval}
\label{sec:Computation of The Costs}
\begin{figure}[ht]
    \centering
    \includegraphics[width=0.65\linewidth]{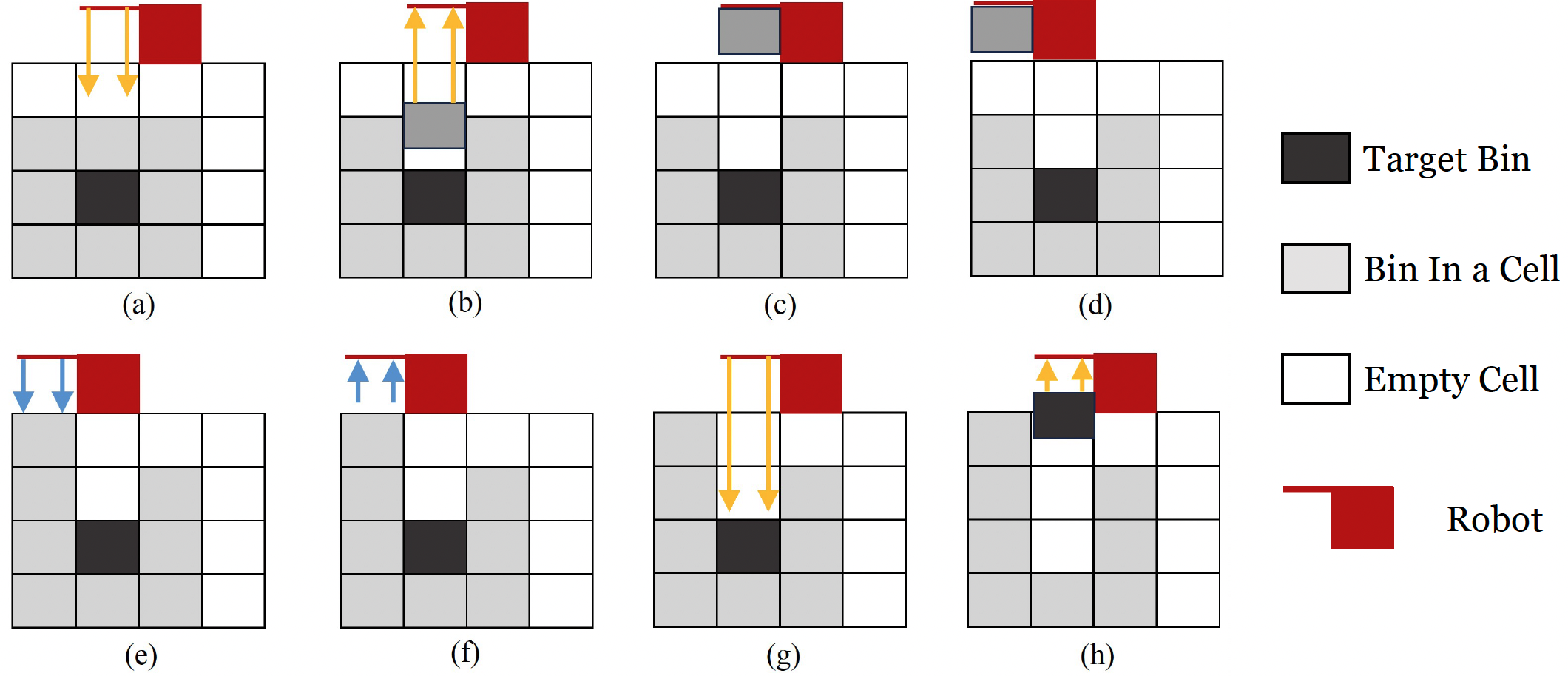}
    \caption{Example illustration of the bin retrieval process (the arrows indicate the gripper motion and direction)}
    \label{fig:intro_costs}
\end{figure}

In an $H$ height grid, to retrieve a bin from layer $l$ in a target stack $m$, a robot needs to dig up bins above the target bin one by one and place them on top of the nearby occupied stack(s) $k\in  \llbracket M\rrbracket/\{m\}$ (see \Cref{fig:intro_costs}). 
Focusing on the gripper motion and excluding the delivery motion, we decompose the bin retrieval time, $C_r$, into two parts:
\begin{enumerate}
    \item $C_{r,1}(l, H, h_m)$: the cost of digging up bins from the target stack, i.e., the cost incurred in the target stack (represented by yellow arrows in \Cref{fig:intro_costs}).   
    \item $C_{r,2}(l, H, h_m, h_k)$: the cost of placing the bins removed above the target bin on top of the nearby stacks, i.e., the costs incurred outside of the target stack (represented by blue arrows in \Cref{fig:intro_costs}).
\end{enumerate}

In addition, all occupied stacks have the same fill level, $h_m=h_k=h_c$.
As a result, the bin retrieval cost is computed by adding the two parts:
\begin{equation}
    \label{eqn:BRC}
    C_r(l, H, h_c) = C_{r,1}(l, H, h_c)+C_{r,2}(l, H, h_c)\textrm{.}
\end{equation}

The cost also depends on some other parameters, such as the dimension of the grid (the height of the bin) and the moving speed of the robot gripper.
To simplify the analysis, we apply a few approximations when deriving the cost functions in \Cref{sec:costFCN_aprox}.

With a given BGC $\mathbf{B}$, we can compute the expected bin retrieval cost for a single bin request $E_1[C_r]$, which is the weighted average of the bin retrieval cost of each bin in the given BGC, by the following equation
\begin{equation}
    \label{eqn:BRC_expected}
    \begin{aligned}
        E_1[C_r]    &=\sum_{l\in  \llbracket H\rrbracket}\sum_{m\in  \llbracket M\rrbracket}C_r(l, H, h_c)\cdot p_{\mathbf{B}_{l,m}} \\
                    &=\sum_{l\in  \llbracket H\rrbracket}\Big(C_r(l, H, h_c)\cdot\sum_{m\in  \llbracket M\rrbracket} p_{\mathbf{B}_{l,m}}\Big) \\
                    &=\sum_{l\in \llbracket H\rrbracket}C_r\Big(l, H, \Omega(\mathbf{B})\Big)\cdot \pi_l(\mathbf{B})\textrm{,}
    \end{aligned}
\end{equation}
where $\pi_l(\mathbf{B})=\sum_{m\in \llbracket M\rrbracket}p_{\mathbf{B}_{l,m}}$ is the probability that the target bin is in layer $l$ within the BGC $\mathbf{B}$. 
Note that $\sum_{l\in \llbracket H\rrbracket}\pi_l(\mathbf{B})=1$.
Since the fill level $h_c$ is a property of a BGC $\mathbf{B}$, we define a function $\Omega(\mathbf{B})=h_c$ that computes the fill level of the BGC.
\subsection{Bin Grid Configuration Problem}\label{lab:BGC problem}
In this paper, our objective is to reduce the bin retrieval time (see \Cref{fig:my_flow}) by reducing the \emph{digging workload} of the robots, which depends on the number of bins above the target bin and the \emph{digging depth} of the target bin. See \Cref{fig:Notation2} where the target bin has a digging depth of three and the number of bins above the target bin is one.
\begin{definition}[Digging Depth]\label{def:digging depth}
    The digging depth of a target bin is given by the layer it occupies within the grid. 
\end{definition}

First, we consider requesting only one bin from the system.
\begin{problem}[Optimal Bin Grid Configuration Problem]\label{Q1:1BGC}
Given a grid consisting of $M$ stacks, where each stack stores up to $H$ bins, and a total of $N (\leq M\cdot H)$ bins to be stored in the grid, find the BGC that minimizes the expected bin retrieval cost for a single bin request, i.e.,
\begin{equation}
    \label{eqn:minexp}
    \begin{aligned}
        &\displaystyle\min_{\mathbf{B}\in \Phi}\ \ && E_1[C_r] =\sum_{l\in \llbracket H\rrbracket}C_r\Big(l, H, \Omega(\mathbf{B})\Big)\cdot \pi_l(\mathbf{B})\\
        &\text{subject to}    && \sum_{l\in \llbracket H\rrbracket}\pi_l(\mathbf{B})=1\textrm{,}\\
                                & && \overline{H}\leq \Omega(\mathbf{B})\leq H\textrm{.}
    \end{aligned}
\end{equation}
We call the solution to (\ref{eqn:minexp}) an \emph{optimal BGC}, which is not unique to a system with multiple occupied stacks.
\end{problem}

\begin{definition}[The Set of All Optimal BGCs]
    The set of all optimal BGCs, $\mathbb{B}\subset \Phi$, is a set that contains all the solutions to (\ref{eqn:minexp}).
\end{definition}

If the target bin is not in the surface layer, after returning the bin and placing it on top of a storage stack, the optimal BGC determined by Problem \ref{Q1:1BGC} is not maintained, which means that the optimal BGC is not sustainable while fulfilling a series of bin requests.
To address this issue, we seek a set of more flexible BGCs that 1) deviate minimally from optimal BGCs and 2) are transformable and maintainable while processing a series of bin requests.
Using the equivalent class relation of a BGC (recall \Cref{rmk: eqc_r} at the end of \Cref{lab:Environment Model}), we define \emph{the set of equivalent optimal BGCs}:
\begin{definition}[The Set of Equivalent Optimal BGCs]\label{def:eq opt bgcs}
    The set of equivalent optimal BGCs $\hat{\mathbb{B}}\subset\Phi$ is a set containing all BGCs that the system aims to transform into, $\hat{\mathbb{B}}=\cup_{\mathbf{B}\in\mathbb{B}}[\mathbf{B}]$.
\end{definition}

Next, we design a policy that transforms any BGC into a BGC that belongs to the set of equivalent optimal BGCs while processing a series of bin requests $\boldsymbol{\sigma}=[\sigma_1, \sigma_2,\dots, \sigma_k,\dots]$.
We use $\sigma_k$ to denote the bin ID of the $k$th requested bin. With a slight abuse of notation, we denote the BGC after returning the $k$th requested bin by $\mathbf{B}^k$ where $\mathbf{B}^0$ is the initial BGC.
We have two possible initial situations, $\mathbf{B}^0\in\hat{\mathbb{B}}$ or $\mathbf{B}^0\notin\hat{\mathbb{B}}$.
To regulate and govern this transformation process, we address and solve the following problem.
\begin{problem}[Bin Rearrangement Problem]
\label{Q2:policy}
    Given an initial BGC $\mathbf{B}^{0} \notin \hat{\mathbb{B}}$ and a sequence of bin requests $\boldsymbol{\sigma}$, find a policy $\Gamma(\cdot)$ such that
    \begin{equation}
        \label{eqn: policy}
        \mathbf{B}^{k+1}=\Gamma(\mathbf{B}^{k}, \sigma_{k+1}) \text{ \ \ \  with \ \ \  } k\in\{0,1,\dots\} \textrm{, }
    \end{equation}
    where
    \begin{enumerate} 
        \item the policy transforms $\mathbf{B}^{0}$ into a $\mathbf{B}^{\lambda}\in \hat{\mathbb{B}}$ in $\lambda$ bin requests where $\forall k < \lambda$ we have $\mathbf{B}^{k}\notin \hat{\mathbb{B}}$;
        \item for every $k\geq \lambda$ the policy ensures that $\mathbf{B}^{k}\in[\mathbf{B}^{\lambda}]\subset\hat{\mathbb{B}}$, i.e., $[\mathbf{B}^{\lambda}]$ and $\hat{\mathbb{B}}$ are each positively invariant under $\Gamma(\cdot)$.
    \end{enumerate}
\end{problem}

Our strategy to solve Problem \ref{Q2:policy} is to gradually adjust the bins in each stack and bring the BGC closer to an equivalent optimal BGC (we provide the fine definition of the distance in \Cref{LGCP} before we measure it). 
We solve Problem \ref{Q1:1BGC} in \Cref{lab: Sec BGC} and Problem \ref{Q2:policy} in \Cref{lab: Sec Policy}.
\section{The Optimal BGC}
\label{lab: Sec BGC}
We first derive the equations to compute the bin retrieval cost in \Cref{sec:costFCN_aprox}.
Then, we present the solution to Problem \ref{Q1:1BGC} in \Cref{sec:slnP1}.
\subsection{Cost Function and Approximation}
\label{sec:costFCN_aprox}
To simplify the calculation, we set the cost of a robot raising or lowering its gripper by one cell height at $1$. We also neglect the cost of a robot to load and unload a bin with its gripper, and the cost is the same whether the gripper is carrying a bin or not regardless of the weight of the bin.
As a result, the cost is $l$ when a robot:
\begin{enumerate}
    \item releases its gripper from the top of the grid to a cell in layer $l$ (e.g., see \Cref{fig:intro_costs}g, where the target bin is located in layer $3$, and the robot releases the gripper to the target bin via the yellow down arrow at a cost of $3$),
    or,
    \item lifts its gripper from a cell in layer $l$ to the top of the grid (e.g., see \Cref{fig:intro_costs}b, where the bin $x$ is located in layer $2$, and the robot lifts the gripper with the bin to the top of the grid via the yellow up arrow at a cost of $2$).
\end{enumerate}

In what follows, we first derive the expressions of $C_{r,1}(l, H, h_c)$ and $C_{r,2}(l, H, h_c)$, respectively. Then, we sum the two results to arrive at the expression for $C_r(l, H, h_c)$.
Before starting the derivations, we specify that the surface layer is located in the layer $h_e+1$ of the grid. Recall that $h_e=H-h_c$ is the number of empty cells in each occupied stack.

The function $C_{r,1}(l, H, h_c)$ computes the cost of digging up the bins above the target bin and the target bin in the target stack. 
The cost of removing the topmost bin (in the surface layer $h_e+1$) from the stack includes releasing the gripper to reach the bin and raising the gripper to bring the bin to the top of the grid, which is $2(h_e+1)$.
Then we add the cost of removing the next topmost bins from the target stack, one at a time, until we get the target bin (in layer $l$) at the top of the grid.
By the summation formula of the arithmetic series, the cost of digging up bins from the target stack is computed by the following equation
\begin{equation}\label{eqn:BRC1}
    \begin{aligned}
        C_{r,1}(l, H, h_c)\equiv C_{r,1}(l, h_e)&=\underbrace{\underbrace{2(h_e+1)}_{\text{the $1^{st}$ topmost bin}}+\underbrace{2(h_e+2)}_{\text{the $2^{nd}$ topmost bin}}+\cdots+\underbrace{2(l)}_{\text{the target bin}}}_{\text{where } [h_e+1, h_e+2,\dots,l]\text{ is an arithmetic series}}\\
        &=2\underbrace{\frac{\Big((h_e+1)+l\Big)\Big(l-(h_e+1)+1\Big)}{2}}_{\text{by the summation formula of the arithmetic series}}\\
        &=l^2+l-h_e^2-h_e\textrm{,}
    \end{aligned}
\end{equation}
which is a monotonically increasing function with respect to $l$. 
    
The function $C_{r,2}(l, H, h_c)$ computes the cost of placing the dug-up bins above the target bin on top of the nearby stack(s). 
Since we consider a single bin request, 
every occupied stack, excluding the target stack, has $h_e+1$ empty cells (the $1$ is the temporary cell above a stack) to temporarily store the dug-up bins.
The costs of filling these empty cells from the bottom to the top with dug-up bins are $[2(h_e),2(h_e-1),\dots,2(0)]$, we call it the \emph{C sequence}. 
We choose to place the $l-h_e-1$ dug-up bins to fill the closest stack (Manhattan distance), then the next closest stack, etc.
Therefore, the cost of placing each of the dug-up bins is a sequence that is formed by cyclically repeating the elements of the C sequence, i.e., $[2(h_e),2(h_e-1),\dots,2(0),2(h_e),\dots]$, and we call it the \emph{D sequence}. 

\begin{figure}[ht]
    \centering
    \includegraphics[width=0.45\linewidth]{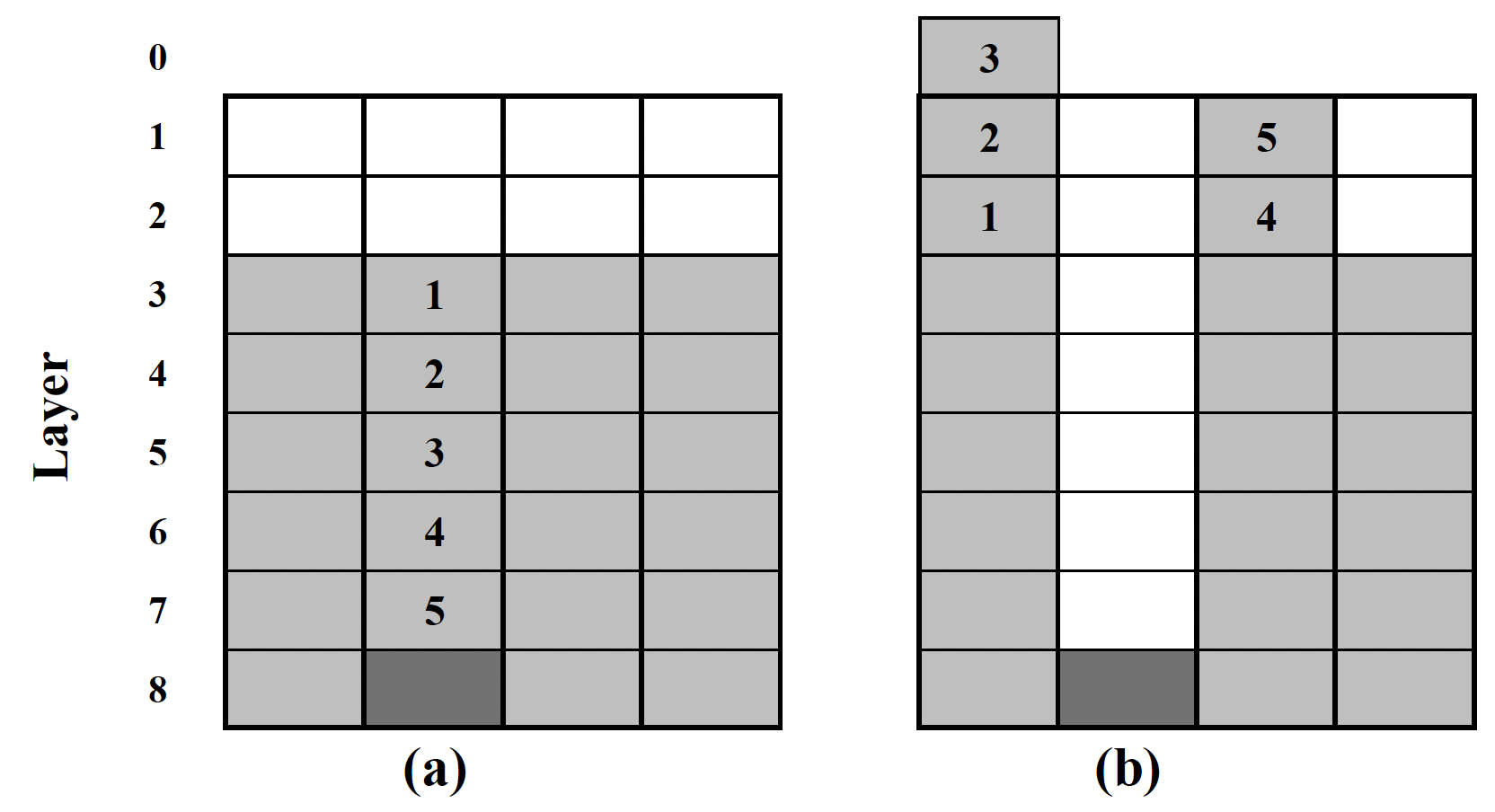}
    \caption{A system: $H=8$, $h_c=6$, $h_e=2$, target bin (dark gray rectangle) is located in $l=8$. Light gray rectangles represent bins and white rectangles are empty cells.}
    \label{fig:eg_Cr2}
\end{figure}
\begin{example}[Computation of $C_{r,2}(l, H, h_c)\equiv C_{r,2}(l, h_e)$]
    \Cref{fig:eg_Cr2}a presents the BGC before bin retrieval begins. 
    Accordingly, $C\ sequence$ is $[4,2,0]$, where $4=2(h_e)=2(2)$ and $D\ sequence$ is $[4,2,0,4,2,\dots]$.
    We need to temporarily store the $5$ bins, where $5=l-h_e-1=8-2-1$.
    \Cref{fig:eg_Cr2}b presents the arrangement of the bins after removing and temporarily storing all the bins above the target bin. Consequently, $C_{r,2}(l=8, h_e=2)=4+2+0+4+2=12$.
    \oprocend
\end{example}

The result of $C_{r,2}(l, h_e)$ is the sum of the first $l-h_e-1$ elements of the D sequence.
Since each element of the D sequence is a non-negative integer, $C_{r,2}(l, h_e)$ is monotonically non-decreasing with respect to $l$.
To quickly compute the value of $C_{r,2}(l, h_e)$, we store all possible results in a lookup table (see Appendix~\ref{APX:LUT}), and define a function $T(h_e,l)$ that computes the values stored in rows named $h_e$ and columns named $l$, i.e.,
\begin{equation}\label{eqn:BRC2}
    C_{r,2}(l, H, h_c)\equiv C_{r,2}(l, h_e)=T(h_e,l)\textrm{.}
\end{equation}

Finally, we compute the bin retrieval cost function by summing (\ref{eqn:BRC1}) and (\ref{eqn:BRC2}) as
\begin{equation}\label{eqn:BRC-complete}
    C_{r}(l,H,h_c)\equiv C_r(l,h_e)=l^2+l-h_e^2-h_e+T(h_e,l)\textrm{,}
\end{equation}
which is a monotonically increasing function with respect to $l$.
\subsection{The Optimal Bin Grid Configuration}
\label{sec:slnP1}
We find the optimal BGCs that solve the Problem \ref{Q1:1BGC} have the following property.
\begin{proposition}[Property of Optimal BGCs]
\label{prp_OPTBGC}
    In an optimal BGC, $\forall l_a,l_b\in\{h_e+1,\dots,H\}$, where $h_e+1$ is the surface layer of the BGC, and $l_a<l_b$, then the popularity of any bin stored in a cell in layer $a$ must be greater than or equal to the popularity of any bin stored in a cell in layer $b$, i.e., $p_{\mathbf{B}_{l_a,m}}\geq p_{\mathbf{B}_{l_b,m}}$.
\end{proposition}

\begin{proof}[By Contradiction.]
    Given the optimal BGC with the empty level $h_e$ for each occupied stack, for all $l\leq h_e$, we have $\pi_l(\mathbf{B})=0$.
    Thus, the objective function of Problem \ref{Q1:1BGC} is rewritten as
    \begin{equation}\label{eqn:E1_mod1}
        E_1[C_r] =  \sum_{l=h_e+1}^{H}C_r(l, h_e)\cdot \pi_l(\mathbf{B})\textrm{.}
    \end{equation}

    \begin{figure}[ht]
        \centering
        \includegraphics[width=0.55\linewidth]{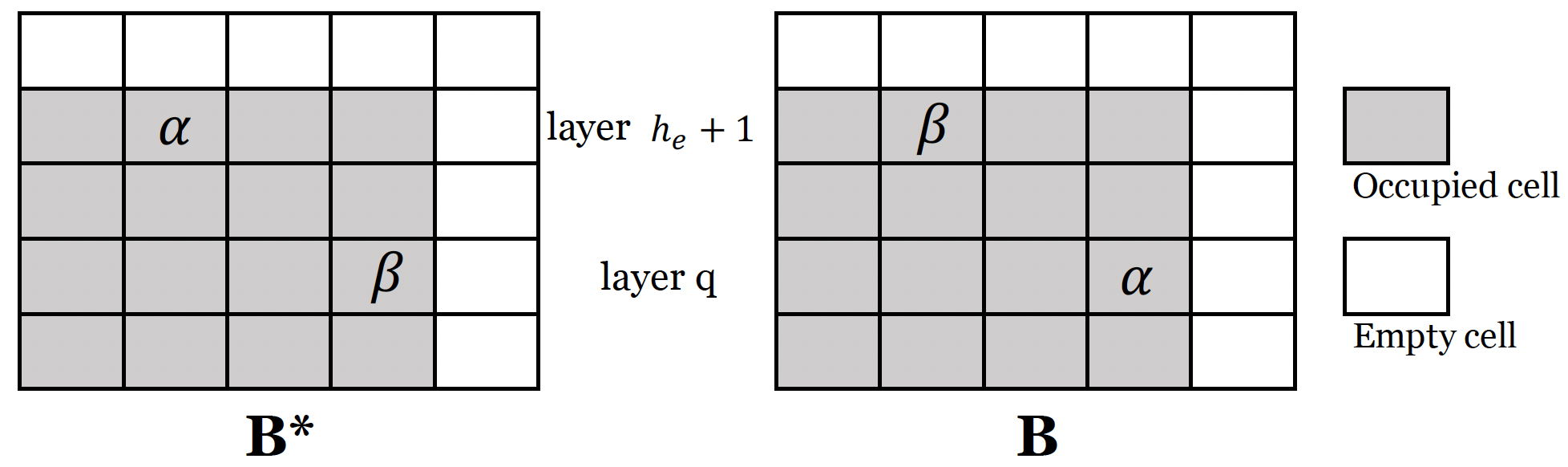}
        \caption{For the proof of the optimal BGC}
        \label{fig:absOPTBGC_example}
    \end{figure}
    
    For a $\mathbf{B}^*$ that satisfies \Cref{prp_OPTBGC}, we denote its expected cost as $E_1^*[C_r]$.
    In $\mathbf{B}^*$, layer $h_e+1$ is the surface layer that stores the $m_f$ most popular bins, and so on for deeper layers.
    
    Assume another $\mathbf{B}$ that achieves the minimum expected bin retrieval cost for a single bin request, but does not satisfy \Cref{prp_OPTBGC}. We denote its expected cost as $E_1[C_r]$.
    The layer $h_e+1$ and layer $q \in\{h_e+2,\dots, H\}$ of $\mathbf{B}$ are different from that of $\mathbf{B^*}$ (see \Cref{fig:absOPTBGC_example}) and given $p_{\alpha}>p_{\beta}$, i.e.,
    
    \begin{equation}
        \begin{cases}
        \begin{aligned}
            & &&\pi_l(\mathbf{B})\neq\pi_l(\mathbf{B}^*) \text{,   } \forall k\in \{h_e+1,q\}\textrm{,}\\
            & &&\pi_l(\mathbf{B})=\pi_l(\mathbf{B}^*) \text{,   } \forall k\in \{h_e+1,\dots, H\}/\{h_e+1,q\}\textrm{.}
        \end{aligned}
        \end{cases}
    \end{equation}
    
    We calculate the difference in expected costs for the two BGCs as follows:
    \begin{equation}
        \begin{aligned}
            & E_1^*[C_r]-E_1[C_r]    \\    
            & = \sum_{l=h_e+1}^{H}C_r(l, h_e)\cdot \pi_l(\mathbf{B}^*)-\sum_{l=h_e+1}^{H}C_r(l, h_e)\cdot \pi_l(\mathbf{B}) \\
            & = \sum_{l=h_e+1}^{H}C_r(l, h_e)\cdot \Big(\pi_l(\mathbf{B}^*)-\pi_l(\mathbf{B})\Big)   \\
            & = C_r(h_e+1, h_e)\cdot \Big(\pi_{h_e+1}(\mathbf{B}^*)-\pi_{h_e+1}(\mathbf{B})\Big) + C_r(q, h_e)\cdot \Big(\pi_q(\mathbf{B}^*)-\pi_q(\mathbf{B})\Big) \\
            & = C_r(h_e+1, h_e)\cdot(p_{\alpha}-p_{\beta})+C_r(q, h_e)\cdot(p_{\beta}-p_{\alpha}) \\
            & = \underbrace{\Big(p_{\alpha}-p_{\beta}\Big)}_{\text{[$>0$, since $p_{\alpha}>p_{\beta}$.]}}\cdot\underbrace{\Big(C_r(h_e+1, h_e)-C_r(q, h_e)\Big)}_{\text{[$<0$, since $C_r$ is an mono increasing equation.]}} \\
            & < 0\\
            &\Rightarrow E_1^*[C_r]<E_1[C_r]\textrm{,}
        \end{aligned}
    \end{equation}
    which conflicts with our assumption.
    Thus, in the optimal BGC, the most popular $m_f$ bins must be in the surface layer. We use this process to gradually verify bins in deeper layers.
\end{proof}

We define the expected bin retrieval cost for a single bin request corresponding to a given empty level as $h_e$ as $E_1[C_r]_{h_e}$.
With the property of optimal BGC and (\ref{eqn:E1_mod1}), we calculate the minimum expected bin retrieval cost $E_1[C_r]_{h_e}$ for a single bin request for all possible empty levels $\forall h_e\in \{0,\dots,H-\overline{H}\}$.
The optimal empty level $h_e^*$ is the $h_e$ that corresponds to the minimum $E_1[C_r]_{h_e}$.
Knowing the optimal empty level and the property of the optimal BGC, we solve Problem \ref{Q1:1BGC}.  
    
However, the optimal BGC is changed and is not optimal if the requested bin is not in the surface layer. 
As a result, the optimal BGC cannot be maintained while serving a series of bin requests at all times.
On the basis of the optimal BGC, we propose a set of more flexible BGCs that can be maintained while serving successive bin requests in the next section.
\section{Policy For Bin Rearrangement}\label{lab: Sec Policy}

In \Cref{sec:dynamic BGC}, we present a method that builds the set of equivalent optimal BGCs. 
In \Cref{LGCP}, we propose a policy that solves Problem \ref{Q2:policy}.

\subsection{The Equivalent Optimal BGCs}\label{sec:dynamic BGC}
Earlier in \Cref{lab:BGC problem}, we defined the set of equivalent optimal BGCs $\hat{\mathbb{B}}$ (\Cref{def:eq opt bgcs}) as a set that includes the equivalent classes of all optimal BGCs. 
By the defined equivalence relation (\Cref{rmk: eqc_r}), a BGC belonging to the set of equivalent optimal BGCs should have the same set of bins in each stack of an optimal BGC. 
Hence, we construct a such BGC $\hat{\mathbf{B}} \in \hat{\mathbb{B}}$ by deforming an optimal BGC $\mathbf{B}^*$ in the following steps:  
\begin{enumerate}
    \item Based on the $\mathbf{B}^*$, we set up \emph{layer groups} with each layer group including the bins in the layer, $l\in\{h_e+\textbf{1}, h_e+\textbf{2}, \dots, h_e+\mathbf{h_c}\}$, and we name the layer groups $1, 2, \dots, h_c$ correspondingly.
    \item Each stack in $\hat{\mathbf{B}}$ contains one bin from each layer group, and the bins in a stack can be arbitrarily arranged. We call such a stack a \emph{layer-complete} stack.
\end{enumerate}

\begin{lemma}[Equivalent Optimal BGC Verification]\label{lemma: verify BGC}
    If all the occupied stacks of a BGC are layer-complete, then the BGC belongs to the set of equivalent optimal BGCs. 
\end{lemma}
\begin{proof}[Direct Proof.]
    Assume that a BGC $\mathbf{B}$ with each occupied stack is layer-complete. 
    We rearrange each stack by stacking the bins in ascending popularity order from bottom to top to form a new BGC $\overline{\mathbf{B}}$, where the resulting $\overline{\mathbf{B}}$ is an optimal BGC (i.e., $\overline{\mathbf{B}}\in\mathbb{B}$).
    According to the equivalence relation (\Cref{rmk: eqc_r}), $\mathbf{B}\in[\overline{\mathbf{B}}]$.
    Since $\overline{\mathbf{B}}\in\mathbb{B}$ and by the definition of the set of equivalent optimal BGCs (\Cref{def:eq opt bgcs}), $[\overline{\mathbf{B}}]\subset\hat{\mathbb{B}}$.
    Therefore, $\mathbf{B}\in\hat{\mathbb{B}}$.
\end{proof}

\begin{remark}[Repeated Popularity]
    If many bins have the same popularity, these bins may be suitable for multiple layers in an optimal BGC. Thus, such a bin could belong to multiple layer groups. 
    However, when verifying whether a stack is layer-complete or not, a bin should be classified into exactly one layer group. 
\end{remark}

\begin{example}[Repeated Popularity]

\begin{figure}[ht]
    \centering
    \includegraphics[width=0.5\linewidth]{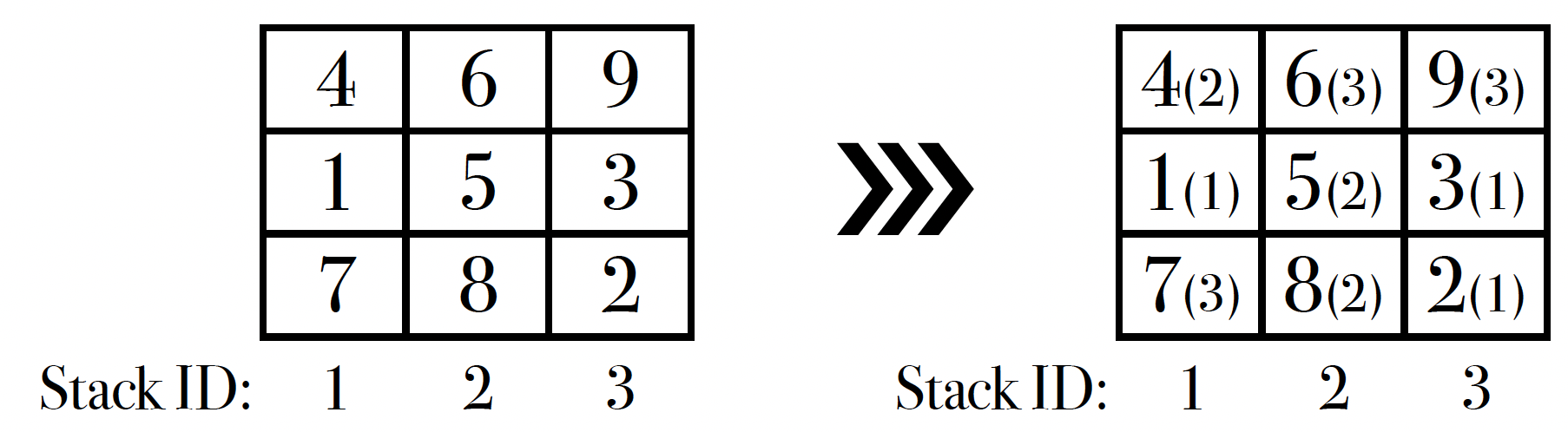}
    \caption{Identifying layer-completed stack}
    \label{fig:LGC-example}
\end{figure}

A grid consists of three stacks ($M=3$), has the height of three cells ($H=3$), and stores nine bins ($N=9$, with bin popularity: $p_1=0.4$, $p_2=0.3$, $p_3=0.06$, and $p_4=\cdots=p_9=0.04$). Bins 1,2 and 3 belong to layer group 1, and the rest of the bins belong to layer groups 2 and 3. 
\Cref{fig:LGC-example} (left) presents a BGC of a system, and the number in each rectangle is the bin ID.

We need to classify each bin into one layer group and verify whether each stack is layer-complete. Then, we can conclude whether the BGC is in the set of equivalent optimal BGCs or not. 

A non-unique result is given in \Cref{fig:LGC-example} (right), and the number in parentheses indicates the layer group into which the bin is classified. 
Only stack $1$ is layer-complete, while stacks $2$ and $3$ are not. 
Thus, the BGC is not in the set of equivalent optimal BGCs.
\oprocend
\end{example}

In the following section, we propose a policy that transforms such a BGC $\mathbf{B}\notin\hat{\mathbb{B}}$ into a BGC $\mathbf{\hat{B}}\in\hat{\mathbb{B}}$ while processing a series of bin requests.

\subsection{The Layer Complete Policy (LCP)}
\label{LGCP}
To transform a BGC that is not in the set of equivalent optimal BGCs into a BGC that is in the set of equivalent optimal BGCs while processing a series of bin requests, we work on selecting appropriate storage stacks for the bins that returned from the workstations. 
Our policy $\Gamma(\cdot)$ takes the current BGC and bin request as input, includes five possible cases when selecting the storage stack, and outputs the BGC after returning the requested bin back to the grid. 

\begin{figure*}[ht]
    \centering
    \includegraphics[width=0.97\textwidth]{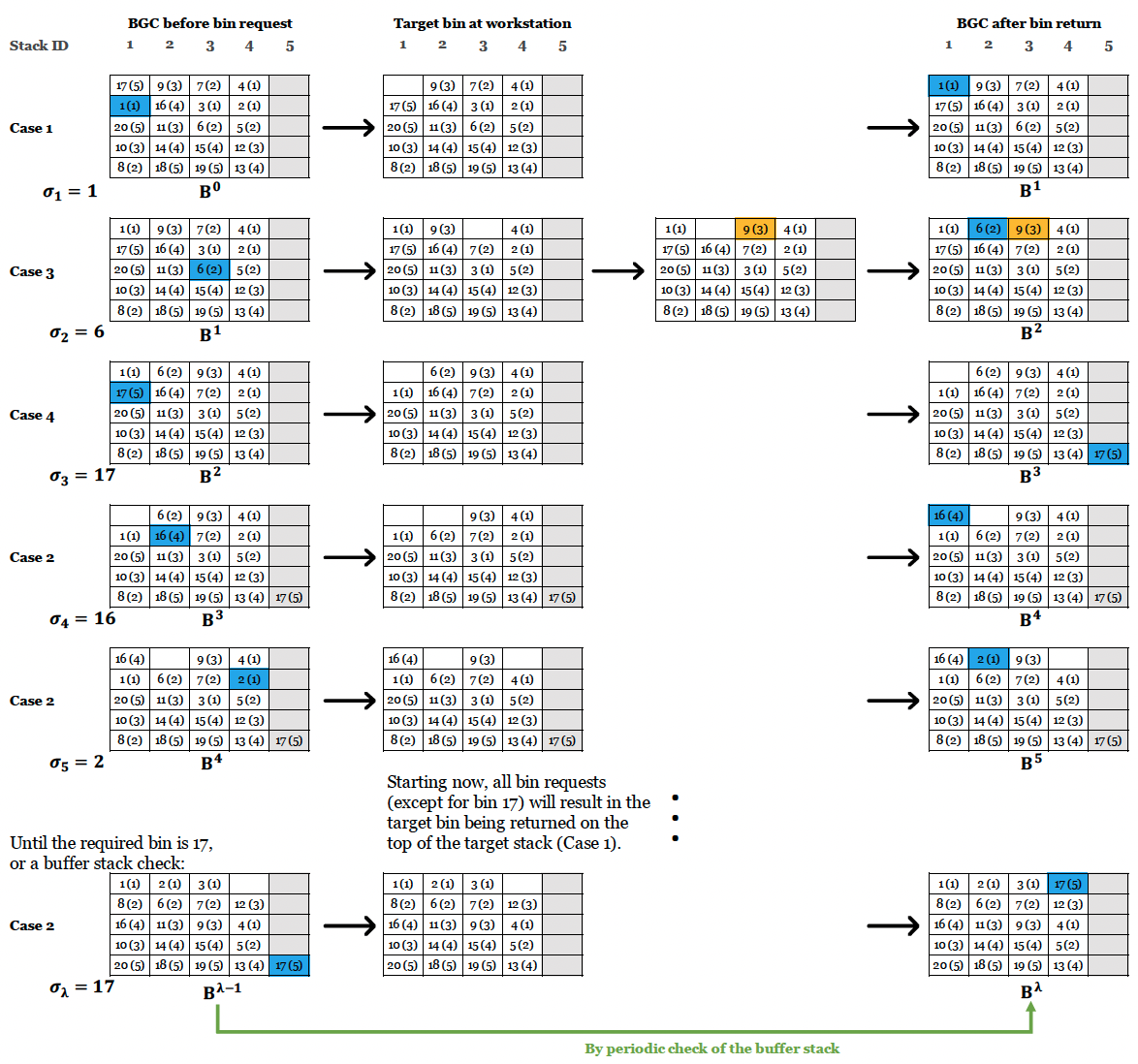}
    \caption{A grid consists of four occupied stacks (white) and a buffer stack (gray), has a height of five cells, and stores twenty bins in total.
    In each small rectangle, the number outside parentheses is the bin ID, and the number inside parentheses is the layer group of the bin. 
    Each figure in the left column is the BGC (together with the request bin are the inputs of $\Gamma(\cdot)$) before a bin retrieval starts, and the blue rectangle in the BGC is the target bin.
    The figures in the middle are BGCs after removing the target bin from the grid. 
    Each figure in the right column is the BGC after returning the target bin to the grid (i.e., the output of $\Gamma(\cdot)$), and the orange rectangle is the swap bin (in case 3). 
    Black arrows indicate the inputs and outputs of the policy. 
    Green arrows indicate the change due to the periodic check of the buffer stack. 
    Step by step, the proposed policy transforms a random BGC $\mathbf{B^0}$ (on top left) into a BGC $\mathbf{B^{\lambda}}$ (on bottom right) that is in the equivalent optimal BGCs.}
    \label{fig:example-policy}
\end{figure*}

In the policy, we need to set aside an empty stack in the grid and name the stack the \emph{buffer stack}.
The cases are checked sequentially and described as follows (see \Cref{fig:example-policy} for examples):
\begin{description}
    \item[\textbf{Case 1:}] If the target bin is the only bin from layer group $x$ in the target stack, then the target bin is placed on top of the target stack.
    \item[\textbf{Case 2:}] 
    If there is another stack $m$ that has no bin from layer group $x$ and has empty cells, then the target bin is placed on top of stack $m$.
    \item[\textbf{Case 3:}] 
    If there is another stack $m$ that no bin from layer group $x$ and has more than one bin (call them the \emph{\textbf{swap bins}}) from layer group $y$, and the target stack has no bin from layer group $y$, then (1) the uppermost swap bin is taken out from stack $m$ and placed on top of the target stack, after, (2) the target bin is placed on top of stack $m$.
    \item[\textbf{Case 4:}] 
    If an empty cell is available in the buffer stack, then the target bin is placed on top of the buffer stack. 
    \item[\textbf{Case 5:}] 
    The target bin is placed on top of the occupied stack that has the maximum number of empty cells.
\end{description}
Additionally, the buffer stack is periodically checked, where robots restore the bins from the buffer stack to occupied stacks with empty cells that need a bin from a matching layer group.

Before assessing the effectiveness of the LCP, we clarify two distances:
\begin{definition}[Distance between Two Stacks]\label{lab: def_d_stacks}
    Given a stack $m$, define a multiset $U_m$ that contains the layer groups of the bins in the stack. The distance between two stacks with stack ID $i$ and $j$ is defined as $D_s(U_i,U_j)=|U_i-U_j|+|U_j-U_i|$.
\end{definition}

\begin{example}[Distance between Two Stacks]
    Given stack $i$ has $U_i=\{1,1,1,2,5\}$ and the layer-complete stack has $U_{*}=\{1,2,3,4,5\}$. 
    Then, the distance between the stack $i$ and the layer-complete stack is $4=|\{1,1\}|+|\{3,4\}|$, where $|\{1,1\}|$ represents the number of abundant bins with repeat layer group(s) and $|\{3,4\}|$ represents the total number of bins for the absent layer group(s) in stack $i$.
    \oprocend
\end{example}

\begin{definition}[Distance between Two BGCs]\label{lab: def_d_BGCs}
    Given two BGCs $\mathbf{B}$ and $\overline{\mathbf{B}}$, the distance between them is defined as the sum of the distances of the corresponding stacks of $\mathbf{B}$ and $\overline{\mathbf{B}}$, i.e., with a slight abuse of notation, we denote $U_{m,\mathbf{B}}$ as the multiset that contains the layer groups of the bins in stack $m$ of $\mathbf{B}$, then $D(\mathbf{B},\overline{\mathbf{B}})=\sum_{m\in \llbracket M\rrbracket}D_s(U_{m,\mathbf{B}}, U_{m,\overline{\mathbf{B}}})$. 
\end{definition}

Given the policy $\Gamma(\cdot)$ (recall (\ref{eqn: policy})), an initial BGC $\mathbf{B}^0$, a series of bin requests $\boldsymbol{\sigma}$, and the corresponding equivalent optimal BGC $\hat{\mathbf{B}}$, we define \emph{the sequence of distances}, $d_k=D(B^k,\hat{\mathbf{B}})$, representing the distance between the \emph{resulting BGC} after returning the $k$th requested bin and the corresponding equivalent optimal BGC.

\begin{proposition}[The Effectiveness of LCP]
    Given any initial BGC and a series of bin requests, under the layer complete policy (LCP), the sequence of distances $d_k$ is monotonically decreasing with $k$.
\end{proposition}
\begin{proof}[Direct Proof.]
    First, we assess each possible change that may happen in a stack by evaluating the impact on the distance between the affected stack and a layer-complete stack:
    \begin{itemize}
        \item [\textit{Removal of a duplicate layer bin (\textbf{RmD}):}] if an occupied stack contains more than one bin from a layer group and one such bin is removed from the stack, the distance decreases by $1$;
        \item [\textit{Insertion of an absence layer bin (\textbf{InA}):}] if an occupied stack does not contain a bin from a layer group and such a bin is placed on top of the stack, the distance decreases by $1$;
        \item [\textit{Insertion of a duplicate layer bin (\textbf{InD}):}] if an occupied stack contains one (or more) bin(s) from a layer group and another such bin is placed on top of the stack, the distance increases by $1$;
        \item [\textit{Free insertion 1 (\textbf{FIn1}):}] if an occupied stack contains only one bin from a layer group, and the bin is taken out and then returned to the same stack, then the distance remains the current value; and
        \item [\textit{Free insertion 2 (\textbf{FIn2}):}] if a bin is placed on top of the buffer stack, the distance remains unchanged.
    \end{itemize}

    Next, we derive for each policy case the corresponding change in the distance between the resulting BGC and an equivalent optimal BGC:
    \begin{itemize}
        \item Case 1: $0$, due to one \textbf{FIn1};
        \item Case 2: $-2$, due to one \textbf{RmD} and one \textbf{InA}; 
        \item Case 3: $-4$, due to two \textbf{RmD}s and two \textbf{InA}s;
        \item Case 4: $-1$, due to one \textbf{RmD} and one \textbf{FIn2};
        \item Case 5: $0$, due to one \textbf{RmD} and one \textbf{InD}; and
        \item Periodic check of buffer stack: $\leq0$, by no or a few \textbf{InA}s.
    \end{itemize}
    
    Since the change in distance is non-positive for each case of the policy and the periodic check, the proposed policy will never increase the distance between the resulting BGC and the equivalent optimal BGC.
    In other words, the policy gradually adjusts the bins in occupied stacks that are not layer-complete, decreasing or maintaining the distance between any stack and a layer-complete stack, and thus the distance between the resulting BGC and an equivalent optimal BGC. 
\end{proof}

To reduce the distance between $\overline{\mathbf{B}}\notin\hat{\mathbb{B}}$ and $\hat{\mathbf{B}}\in\hat{\mathbb{B}}$, the policy must trigger one of \textbf{Cases} 2, 3, or 4. 
In \textbf{Case} 2, one bin request rearranges one bin that needs to be reorganized (calling such a bin \emph{out-of-place} bin).
In \textbf{Case} 3, one bin request rearranges two out-of-place bins. 
In \textbf{Case} 4, one out-of-place bin is moved to the buffer stack, and we assume that the periodic check of the buffer stack moves the bin to its right place, thus one bin request adjusts one out-of-place bin.
Therefore, each out-of-place bin must be requested at least once to make it \emph{in-place}.
Note that the series of bin requests $\boldsymbol{\sigma}$ depends on the demand distributions of the storage items, i.e., the bin popularity $p_n$.

\begin{claim}\label{lab: claim1}
    Suppose all out-of-place bins have a popularity greater than zero, i.e., $p_n>0$ where $n$ is the ID of an out-of-place bin. 
    Then, the LCP can transform any initial BGC $\overline{\mathbf{B}}\notin\hat{\mathbb{B}}$ into an equivalent optimal BGC $\hat{\mathbf{B}}$, and the expected number of bin requests to complete the transformation is finite.
    Thus, the LCP guarantees that any initial BGC converges to an equivalent optimal BGC.
\end{claim}
\begin{proof}[Direct Proof]
    An initial BGC $\overline{\mathbf{B}}\notin\hat{\mathbb{B}}$ has a total number of $x$ out-of-place bins.
    We define a vector $\mathbf{p}$ that stores the popularity of out-of-place bins, where $\mathbf{p}(i)$ represents the $i$th entry of $\mathbf{p}$ and $\mathbf{p}(i)>0$.
    We need to retrieve each of the out-of-place bins at least once to transform a $\overline{\mathbf{B}}$ into a $\hat{\mathbf{B}}\in\hat{\mathbb{B}}$, which is a coupon collector problem (CCP) with different probabilities. 

    Depending on the probability summation of all in-place bins, we have two possible CCP setups:
    \begin{description}
        \item [Setup 1:] if $1-\sum_{i=1}^{x}\mathbf{p}(i)>0$, then we assume there are $x+1$ types of coupons, and that each day a collector randomly gets a coupon with probability $p{'}_i$ corresponding to the $i$th coupon ($1 \leq i \leq x+1$), where 
        \begin{equation}
            p{'}_i=\begin{cases}
                    \begin{aligned}
                    &  &&\mathbf{p}(i) &&& \text{if $1 \leq i \leq x$}\\
                    &  &&1-\sum_{j=1}^{x}\mathbf{p}(j) &&& \text{if $i = x+1$}
                    \end{aligned}
                    \end{cases} 
                    \textrm{.}
        \end{equation} 
        \item[Setup 2:] if $1-\sum_{i=1}^{x}\mathbf{p}(i)=0$, then we assume there are $x$ types of coupons, and that each day a collector randomly gets a coupon with probability $p{'}_i$ corresponding to the $i$th coupon ($1 \leq i \leq x$), where $p{'}_i=\mathbf{p}(i)$.
    \end{description}
    
    Note that in both setups, $p{'}_i>0$ for all $i$.
    Then, we submit the values in the equation that computes the expected time ($C_j$, the number of bin requests in our problem) for a full collection as
    \begin{equation}
        E[C_j]=\sum_{q=0}^{y=x \text{ or } x-1}(-1)^{y-q}\sum_{|J|=q}\frac{1}{1-P_J} \text{\ \ \ with \ \ \ } P_J=\sum_{j\in J}p{'}_j\textrm{.}
    \end{equation}
    Since every $p{'}_i>0$, we have all $0<P_J<1$, which result in $1/(1-P_J)<\infty$. Thus, $E[C_j]$ is finite, which is the expected time to collect all the $x+1$ or $x$ coupons.
    In \textbf{Setup 1}, our goal is to collect $x$ types of coupons (i.e., the out-of-place bins) in the $x+1$ types of coupons, the expected number of bin requests ensuring that each of the out-of-place bins is requested at least once is theoretically smaller than $E[C_j]$.
    Therefore, with the given condition, the excepted number of bin requests to transform $\overline{\mathbf{B}}$ into $\hat{\mathbf{B}}$ is finite, where the number is smaller than (with Setup 1) or equal to (with Setup 2) $E[C_j]$.   
\end{proof}

The above proof also shows that the requirement that all out-of-place bins have positive popularity is also necessary.
If we have any out-of-place bin(s) that has zero popularity, then one of a few $p{'}_i$ is(are) equal to zero.
Therefore, there must exist one (or more) $P_J=1$ and result in $1/(1-P_J)\to\infty$. Thus, $E[C_j]$ is infinite.
Therefore, the LCP does not guarantee the transformation from any BGC into an equivalent optimal BGC in a finite number of bin requests.

In fact, \emph{popular} bins with higher polarities are requested most of the time, and the rest of the bins (\emph{unpopular}) are requested occasionally. The $80/20$ rule - the most popular $20\%$ of bins are retrieved for $80\%$ of order fulfillment - is widely applied in warehouse management \citep{magestoreweb}. 
If we focus on the percentage of popular bins out of all bins ($\varepsilon$) for most of the order fulfillment, we only need to ensure that a BGC can be transformed into a \emph{quasi-equivalent optimal BGC} (see Approximation \ref{lab:quasi_optbgc}).
\begin{approximation}[Quasi-Equivalent Optimal BGC]\label{lab:quasi_optbgc}
    A quasi-equivalent optimal BGC $\tilde{\mathbf{B}}$ has each occupied stack containing one bin from the layer groups $1,\dots,\lceil h_c\cdot\varepsilon\rceil$.
\end{approximation}
Consequently, each popular bin has positive popularity and the sum of popularities of the unpopular bins is greater than $0$.

\begin{claim}
    For any given initial BGC $\overline{\mathbf{B}}\notin\hat{\mathbb{B}}$, under the LCP, the expected number of bin requests to transform $\overline{\mathbf{B}}$ into a quasi-equivalent optimal BGC $\tilde{\mathbf{B}}$ is finite.
\end{claim}
The proof is similar to the proof of Claim \ref{lab: claim1} by setting $\mathbf{p}$ to include the popularity of all popular out-of-place bins and applying setup 1.

After transforming the BGC into a quasi-equivalent optimal BGC, the LCP ensures that every new resulting BGC is still a quasi-equivalent optimal BGC.
In other words, the set of all quasi-equivalent optimal BGCs should be positively invariant under LCP.
Consequently, we modify the equivalent class of a BGC (\Cref{rmk: eqc_r}) to the \emph{quasi-equivalent class of a BGC} as follows.
\begin{remark}[quasi-equivalent class of a BGC]
    In the set of all feasible BGCs $\Phi$, we define the equivalence relation $R_{\varepsilon}$ as follows:
    given $\tilde{\mathbf{B}}\in\Phi$ and $\mathbf{B}\in\Phi$, then $\mathbf{B}R_{\varepsilon}\tilde{\mathbf{B}}$ if and only if for each column $i\in\{1,\ldots,m_f\}$ the $[h_e+1,\dots,h_e+\lceil h_c\cdot\varepsilon\rceil]$ entries of $\tilde{\mathbf{B}}$ exist in the corresponding column of $\mathbf{B}$ (with possibly different ordering).
    We let $[\tilde{\mathbf{B}}]_{R_{\varepsilon}}=\{\mathbf{B}\in\Phi|\mathbf{B}R_{\varepsilon}\tilde{\mathbf{B}}\}$ denote the quasi-equivalence class to which $\tilde{\mathbf{B}}$ belongs.
\end{remark}

In summary, by retrieving bins from and returning bins to the grid, the proposed policy
(i) reduces or maintains the distance between the resulting BGC after each bin request and an equivalent optimal BGC $\hat{\mathbf{B}}$ and (ii) ensures that any BGC can be transformed into a quasi-equivalent optimal BGC $\tilde{\mathbf{B}}$ in a finite number of bin requests.
After the first time the system achieves $\hat{\mathbf{B}}$ or $\tilde{\mathbf{B}}$, the BGC persists in $[\hat{\mathbf{B}}]$ or $[\tilde{\mathbf{B}}]_{R_{\varepsilon}}$ until there is a change in the distribution of bin popularity. 
Subsequently, the proposed policy will, again corresponding to the updated bin popularity, transform the current BGC into an updated quasi-equivalent optimal BGC (guaranteed) or/and an updated equivalent optimal BGC (unguaranteed).
\section{Results and Discussions}
\label{lab: Sec result}
In this section, we evaluate and compare the proposed policy with two existing methods.
\Cref{Implementation Details} introduces the 3D discrete event simulation environment built for policy validation and evaluation. 
We consider scenarios starting with three initial BGCs (one optimal and two non-optimal) and
evaluate three global objectives:
the digging workload in \Cref{Digging Workload},
the bin retrieval time in \Cref{Bin Retrieval Time},
and the robot working time in \Cref{Robot Runtime}.

\subsection{Implementation Details}
\label{Implementation Details}
We built simulation models in FlexSim (version 23.0.3).
The parameters of the storage bin and the robot are shown in \Cref{System Parameters}. 
All simulations are carried out using the system parameterized as shown in \Cref{System Setup}. 
The grid uses a rectangular footprint, and each stack has the same number of cells in our model. 
We use red-floor stacks to represent workstations located on the perimeter of the bottom layer of the system (see \Cref{fig:topview}). 
The bin requests are generated by the distribution presented in \Cref{fig:binP}, which depends on the popularity of the bins. 
To exclude the influence of different warehouse management systems (WMS), we do not streamline bin requests (recall Assumption \ref{lab: asp bin request}).

\begin{table}[ht]
\caption{System Parameters} \label{System Parameters}
\small
\centering
\begin{tabularx}{\columnwidth}{@{}XXXXXX@{}}
\toprule
\multicolumn{2}{c}{Storage Bin} & \multicolumn{4}{c}{Robot} \\
\cmidrule(r){1-2} \cmidrule(l){3-6}
Length           & $0.65\ m$   & Top Speed     & $3.1\ m/s$     & Load     &$1.2\ s$\\
Width            & $0.45\ m$    & Acceleration  & $0.8\ m/s^2$   & Unload   &$1\ s$\\
Height           & $0.33\ m$    & Lift Speed    & $1.6\ m/s$     & Turn     &$1\ s$\\
\bottomrule
\end{tabularx}
\end{table}

\begin{table}[ht]
\caption{System Setup} \label{System Setup}
\small
\centering
\begin{tabularx}{\columnwidth}{@{}XXXXXXX@{}}
\toprule
\multicolumn{2}{c}{Grid Size (cell)} & \multicolumn{2}{c}{Number of Components} & \multicolumn{3}{c}{Orders} \\
\cmidrule(r){1-2} \cmidrule(l){3-4} \cmidrule(l){5-7}
Length  & $24$  & Robots        & $12$      & \multicolumn{2}{l}{Bin Request Rate}    & $5\ requests/min$ \\
Width   & $12$  & Workstations  & $6$       & \multicolumn{2}{l}{Bin Processing Time} &  \\
Height  & $10$  & Bins          & $2730$    & \multicolumn{2}{l}{at Workstation}      & $30\ s/order$ \\
\bottomrule
\end{tabularx}
\end{table}

\begin{figure}[ht]
    \begin{subfigure}{0.6\linewidth}
         \centering
         \includegraphics[width=\textwidth]{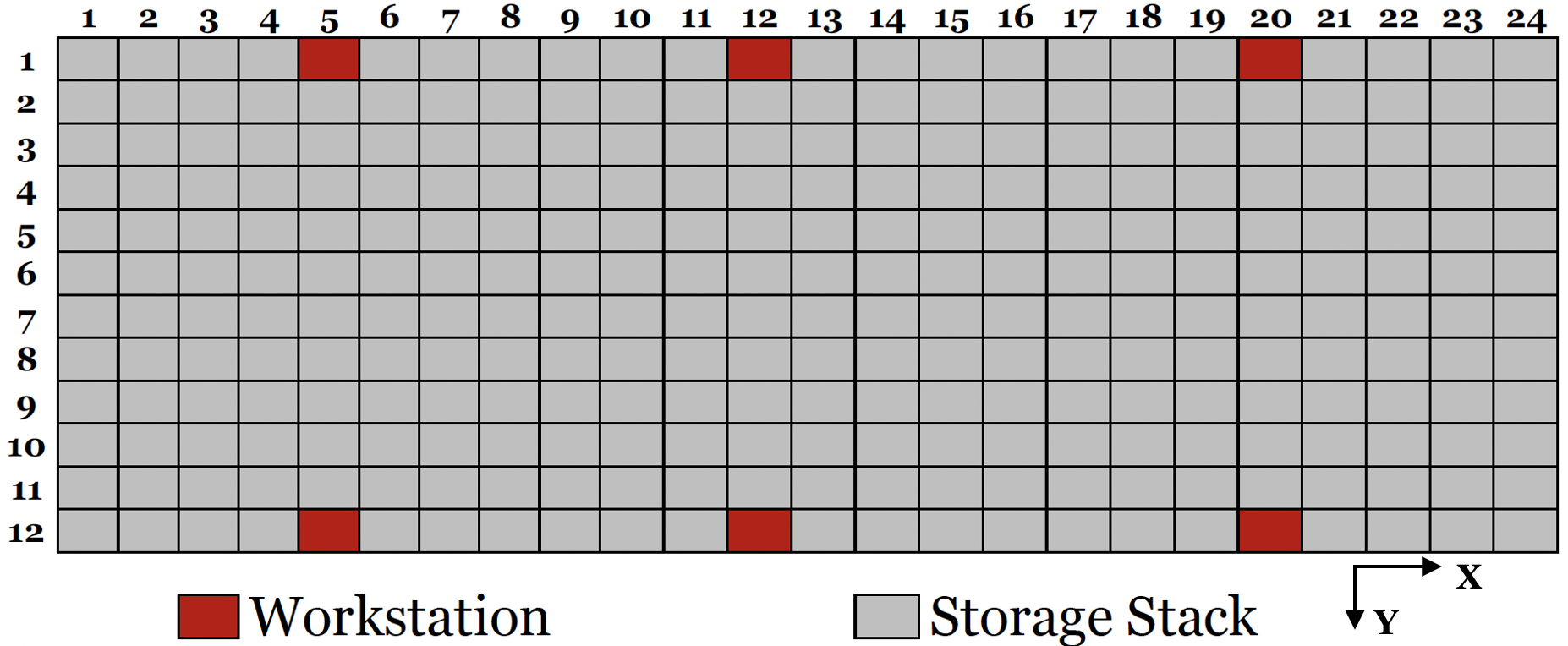}
         \caption{Top view of the grid}
         \label{fig:topview}
     \end{subfigure}
     \hfill
     \begin{subfigure}{0.38\linewidth}
         \centering
         \includegraphics[width=\textwidth]{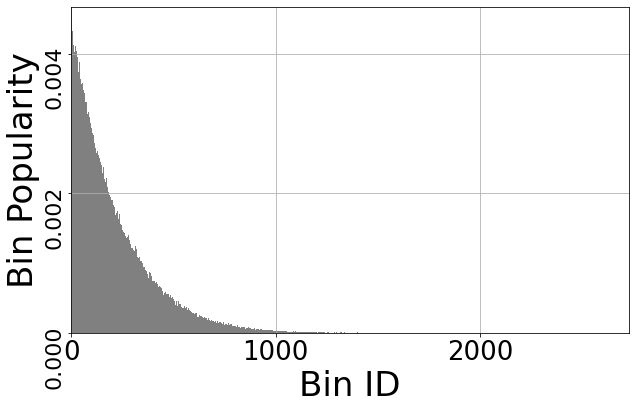}
         \caption{Bin popularity}
         \label{fig:binP}
     \end{subfigure}
    \caption{Simulation environment setups}
    \label{fig: system setup}
\end{figure}

\begin{figure*}[ht]
     \centering
     \begin{subfigure}{0.325\textwidth}
         \centering
         \includegraphics[width=\textwidth]{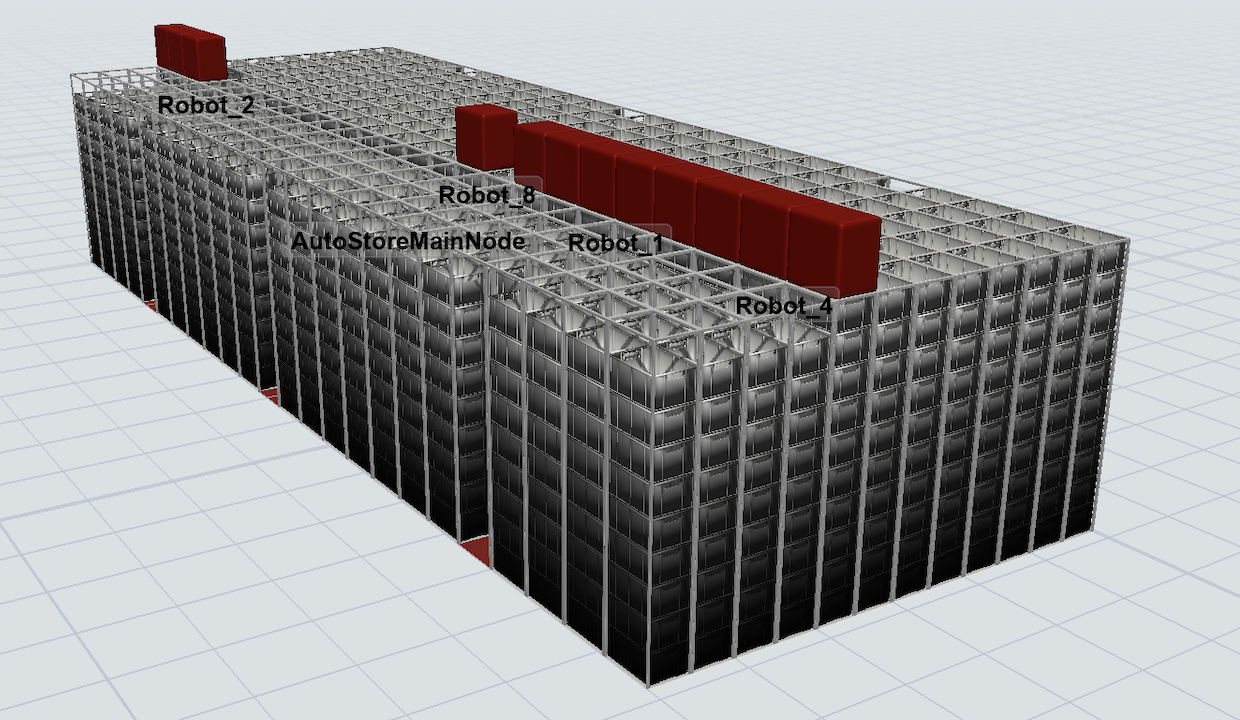}
         \caption{$0\%$ randomization from the optimal BGC (BGC/0)}
         \label{fig:0init}
     \end{subfigure}
     \hfill
     \begin{subfigure}{0.325\textwidth}
         \centering
         \includegraphics[width=\textwidth]{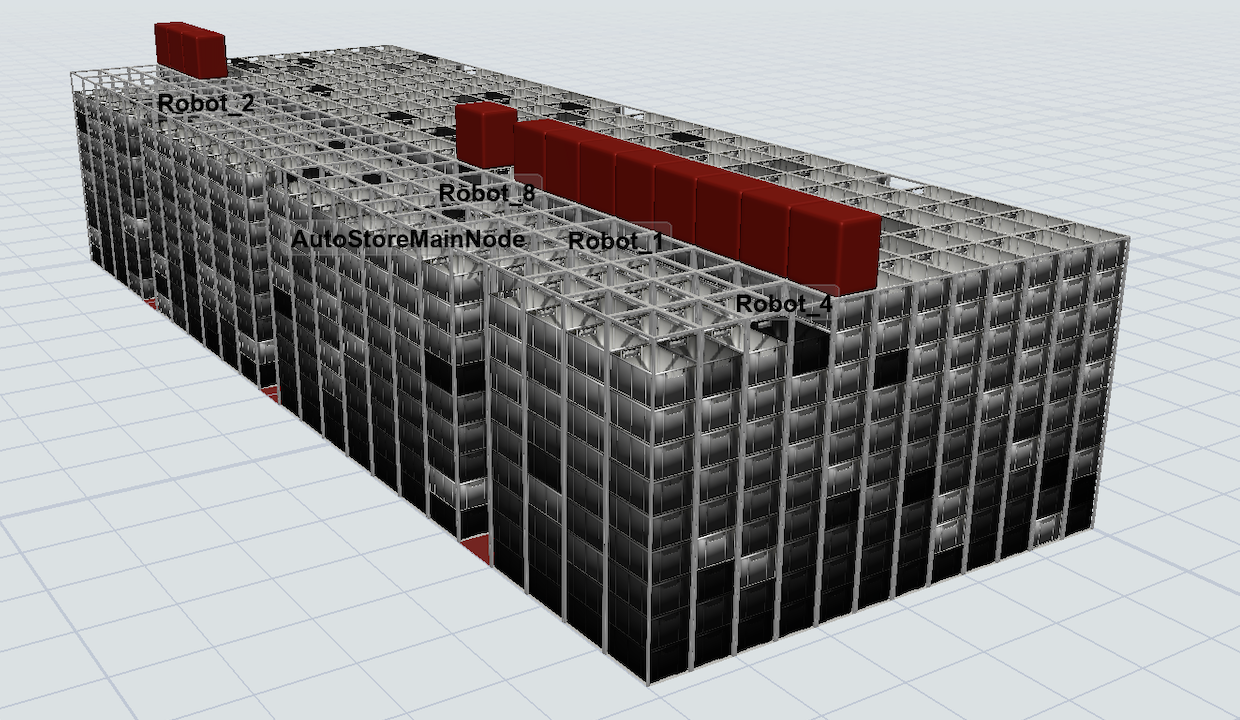}
         \caption{$40\%$ randomization from the optimal BGC (BGC/40)}
         \label{fig:40init}
     \end{subfigure}
     \hfill
     \begin{subfigure}{0.325\textwidth}
         \centering
         \includegraphics[width=\textwidth]{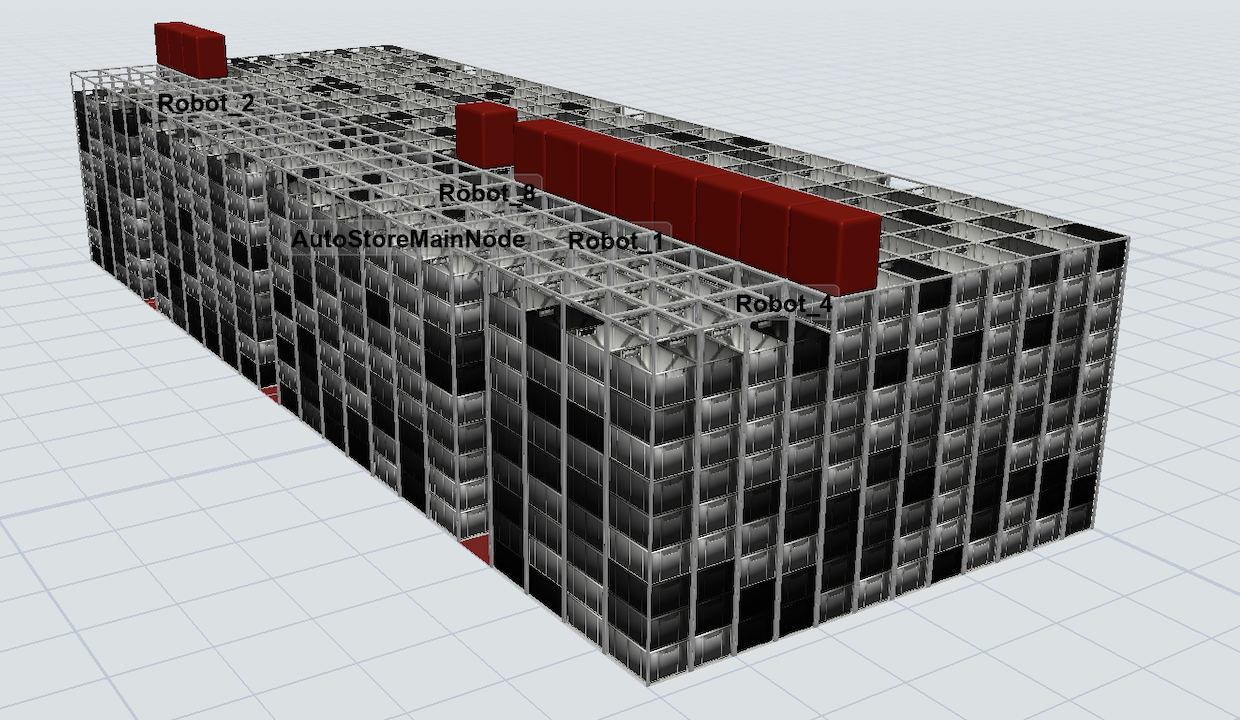}
         \caption{$100\%$ randomization from the optimal BGC (BGC/100)}
         \label{fig:100init}
     \end{subfigure}
    \caption{Initial states with different randomization from the optimal BGC}
    \label{fig: 3D Sim Example}
\end{figure*}

We generate different initial BGCs for the simulation by taking an optimal BGC and swapping selected pairs of bins. 
The pairs are selected uniformly at random and without replacement.
We denote the initial BGC by the percentage of bins that have been swapped in an optimal BGC. 
For example, in a system that stores $100$ bins, an initial $40\%$ swap corresponds to swapping $20$ pairs of bins ($40$ bins in total).
We call this the \emph{randomization from optimal BGC}, with $0\%$ corresponding to the optimal BGC and $100\%$ corresponding to each bin involved in a swap.
\Cref{fig: 3D Sim Example} presents three initial BGCs corresponding to the three scenarios discussed in this section.

We integrate the system with the proposed policy and the delayed and immediate policies (i.e., the shared storage policy coupled with random storage stacks using delayed and immediate reshuffling methods in \citealt{zou2018operating}).
In each combination of scenarios and policies, we simulate the model for $100$ hours and use the result for the following analysis.

This paper focuses on the effect of digging depth of required bins on bin retrieval time.
The location of the workstations is not our object of study,
and in our simplified simulation model:
\begin{enumerate}
    \item from the start point to the destination, robots first run in X direction and then in Y direction and do not perform collision avoidance between robots and bins,
    \item when a new task arrives, the first available robot in the queue accepts the task but not necessarily the nearest one, and
    \item a robot releases the bin at the workstation immediately when it arrives at the top of the workstation.
\end{enumerate} 

When evaluating the performance of the system, specifically the digging time to retrieve the target bin, we do not count the time consumption of robots moving from the dwelling point to the task location (e.g., stacks or workstations).
These simplifications allow us to run simulations on larger systems.  
In fact, our simulation results have shown that they do not affect the general observed trends (see the second half of \Cref{Digging Workload}).

Note that deadlock can occur in AS/RS where two or more entities (e.g., robots) become mutually blocked and unable to proceed further, resulting in a halt or significant delay in system operations.
This problem has been well discussed in AVS/RS \citep{he2009deadlock, roy2013blocking,carlo2012sequencing}.
In RCS/RS, blocked stacks and bins (recall \Cref{lab: the process}) may cause deadlocks in the system. By prioritizing the jobs in order (return $>$ reshuffle $>$ retrieval), we assign free robots to the higher-rank jobs first to prevent the system from deadlock.
With the setup and scenarios in our simulations, deadlock does not occur in our model (the result is not specifically discussed in the paper).
\subsection{Digging Workload}\label{Digging Workload}
In each of the three scenarios (BGC/0, BGC/40, and BGC/100), we compare the digging workload with the three policies.
\Cref{fig: digging depth} (resp. \Cref{fig: nbins above}) presents the distribution of the digging depth of (resp. the number of bins above) the target bins.
\begin{figure*}[ht]
    \centering
    \begin{subfigure}{0.31\textwidth}
        \centering
        \includegraphics[width=\textwidth]{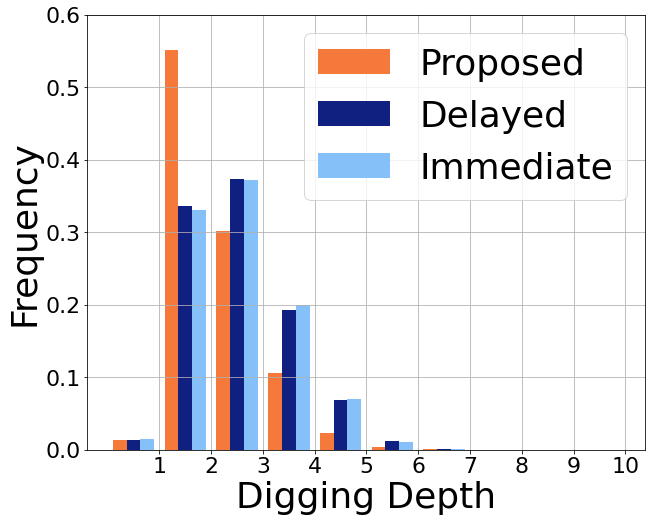}
        \caption{BGC/0}
        \label{fig:0 digdepth}
    \end{subfigure}
    \hfill
    \begin{subfigure}{0.31\textwidth}
        \centering
        \includegraphics[width=\textwidth]{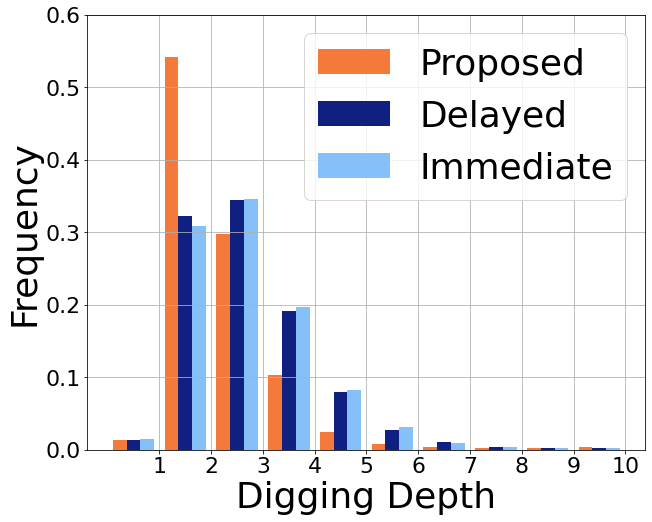}
        \caption{BGC/40}
        \label{fig:40 digdepth}
    \end{subfigure}
    \hfill
    \begin{subfigure}{0.31\textwidth}
        \centering
        \includegraphics[width=\textwidth]{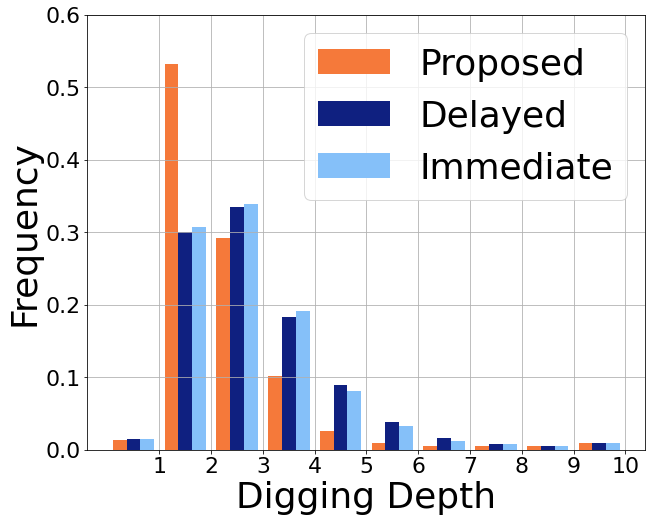}
        \caption{BGC/100}
        \label{fig:100 digdepth}
    \end{subfigure}
    \caption{Comparisons of the digging depth of the target bins.}
    \label{fig: digging depth}
\end{figure*}
\begin{figure*}[ht]
    \centering
    \begin{subfigure}{0.31\textwidth}
        \centering
        \includegraphics[width=\textwidth]{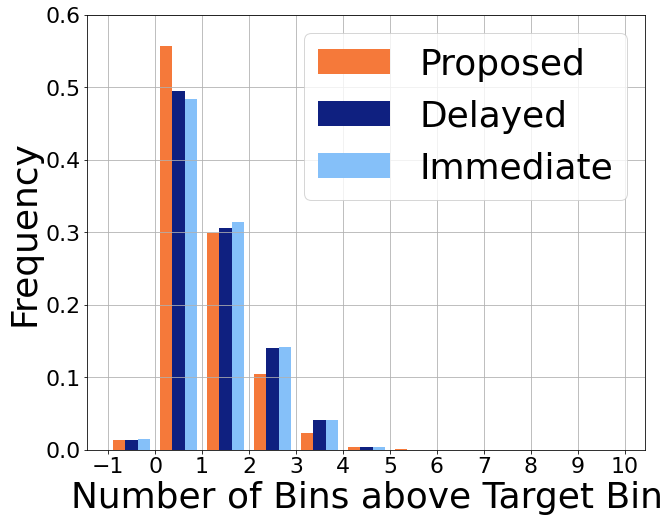}
        \caption{BGC/0}
        \label{fig:0 abovebins}
    \end{subfigure}
    \hfill
    \begin{subfigure}{0.31\textwidth}
        \centering
        \includegraphics[width=\textwidth]{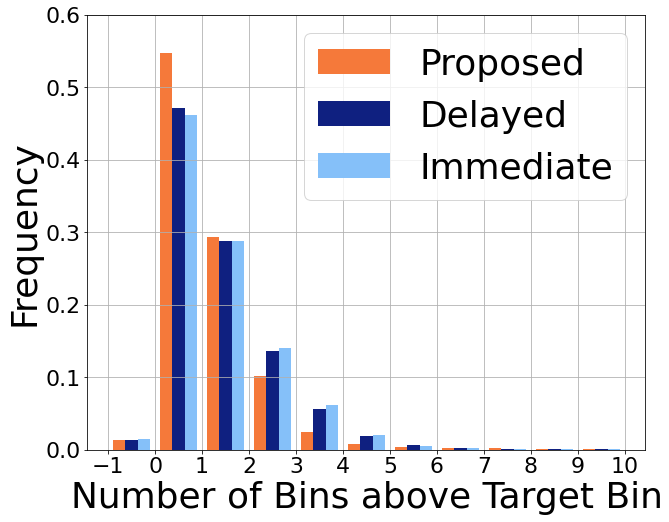}
        \caption{BGC/40}
        \label{fig:40 abovebins}
    \end{subfigure}
    \hfill
    \begin{subfigure}{0.31\textwidth}
        \centering
        \includegraphics[width=\textwidth]{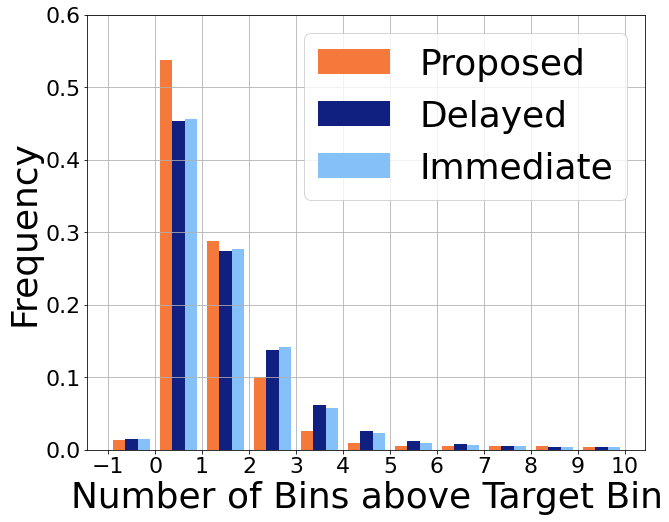}
        \caption{BGC/100}
        \label{fig:100 abovebins}
    \end{subfigure}
     
    \caption{Comparisons of the number of bins above the target bins.}
    \label{fig: nbins above}
\end{figure*}

The first group of bars in each histogram encapsulates the target bins that have already been placed at the workstation when requested. Thus, these bin requests do not trigger any robot task.
Comparing the target bins that are from the top layer (resp. with zero bin above the target bin), we find that the portion in the proposed policy is at least 65\% (resp. 16\%) greater than in delayed and immediate policies. 
As should be, the proposed policy has a smaller number of target bins dug up from the deeper layers (resp. smaller number of bins above the target bin) than the other two policies.

By applying the proposed policy, more than $50\%$ of the requested bins are retrieved from the surface layer, which means that robots do less digging work to retrieve the target bins.
To support this finding and to verify that the impact of the simplifications made in the model (mentioned in \Cref{Implementation Details}) is negligible, we partition the bin retrieval time into four parts:
(1) \emph{waiting time} is the time of a task waiting for a robot to be assigned and the time of a blocked target stack or bin waiting to be freed;
(2) \emph{delivery1 time} is the time of a robot moving from the dwelling point to the target stack; 
(3) \emph{digging time} is the time from the start of removing the first bin out of the target stack to bring the target bin to the top of the target stack; and 
(4) \emph{delivery2 time} is the time of a robot moving from the target stack to the workstation. 
The results illustrate that the proposed policy has
\begin{enumerate}
    \item a significant advantage in reducing the digging time (see \Cref{fig: partition dig} );
    \item almost no impact on the waiting time (all plots are overlapped, omitted); and
    \item almost no impact on the delivery motion, i.e., delivery1 time and delivery2 time (all plots are overlapped, omitted).
\end{enumerate}

\begin{figure}[ht]
    \centering
    \includegraphics[width=0.5\linewidth]{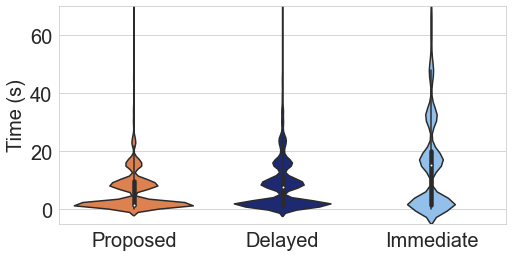}
    \caption{Dig Time Distribution (BGC/40)}
    \label{fig: partition dig}
\end{figure}

In summary, the proposed policy effectively reduces the digging workload by keeping the target bin closer to the top of the grid.
\subsection{Bin Retrieval Time}\label{Bin Retrieval Time}

\begin{figure}[ht]
    \centering
    \includegraphics[width=0.55\linewidth]{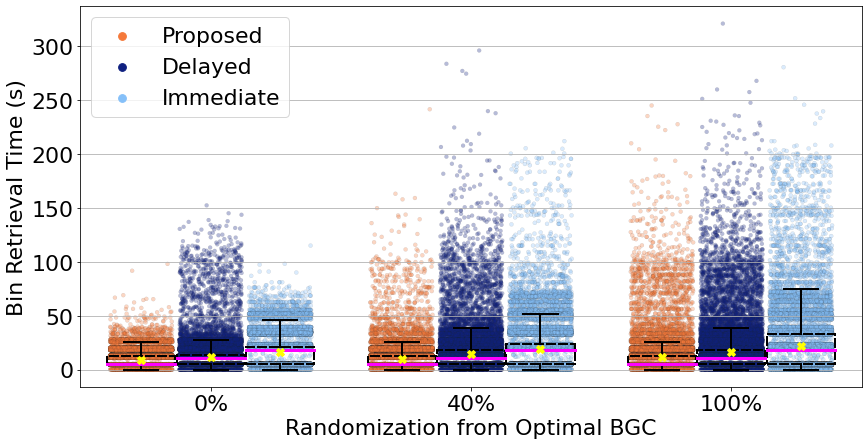}        
    \caption{The distribution of the bin retrieval time - All Bin Requests}
    \label{fig:BRT_F}
\end{figure}

\Cref{fig:BRT_F} presents the bin retrieval time of all bin requests (approximately $30000$ bin requests) in the 100-hour simulation for each of the three policies in three scenarios.
We observe that when the system uses the proposed policy, the bin retrieval times are more concentrated at lower values in all three scenarios compared to the other two policies.
We superimpose each set of data points with a box plot, where the central bright solid line indicates the median, the bottom and top edges of the box indicate the $25$th and $75$th percentiles, and the whiskers extend to the most extreme data points ($1.5$ IQR).

\begin{figure*}[ht]
    \centering
    \includegraphics[width=\textwidth]{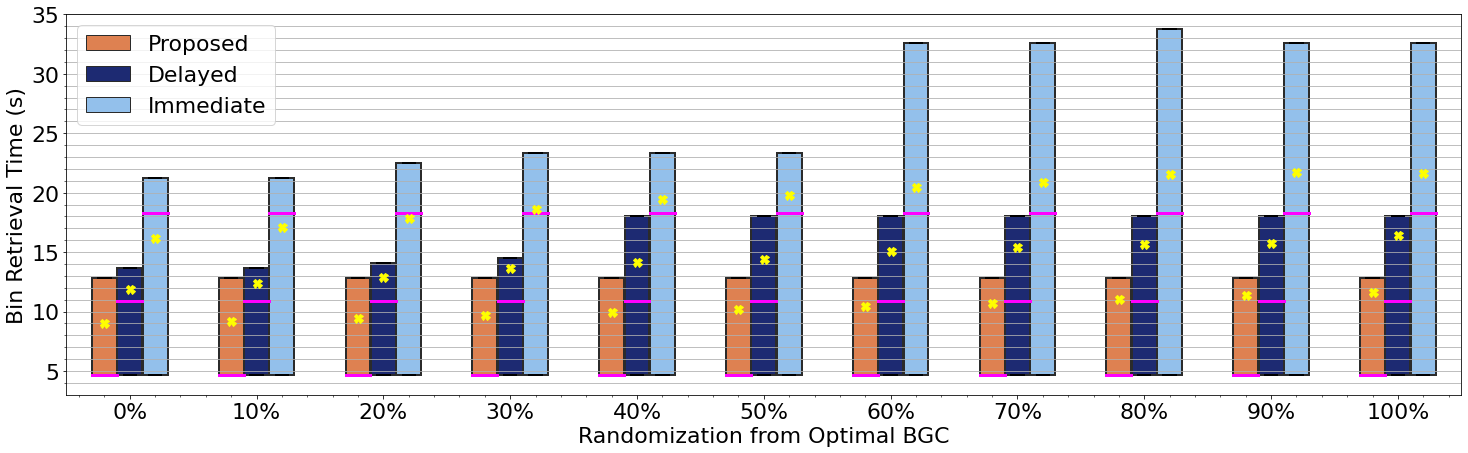}
    \caption{The trend of bin retrieval time -- IQR (box), medians (bright solid line), and means (\textbf{x}) -- versus the initial state with randomization from the optimal BGC.}
    \label{fig:BRT-IQR-FullRange}
\end{figure*}

First, we look at the interquartile range ($25\%-75\%$) of the simulation results presented in \Cref{fig:BRT-IQR-FullRange}, which contains more scenarios (i.e., initial BGCs) that have different randomization from an optimal BGC (from $0\%$ to $100\%$ in increments of $10\%$).
In all scenarios, the proposed policy results in the smallest IQR, and the IQR remains the same.
The IQRs of the delayed and immediate policies extend as the initial randomization from optimal BGCs increases.
The median of each policy remains the same in all scenarios. 
Note that the median for the proposed policy overlaps with the $25$th percentile. 
As the randomness of the initial BGC increases, the mean value under all three policies rises. Among them, the lowest mean value and rate of increase are seen in simulations using the proposed policy.
These results indicate that the proposed policy is effective in controlling the bin retrieval time to a small fluctuation, thus ensuring a more stable bin retrieval time compared to the other two policies. 

\begin{table*}[ht]
    \caption{Improvement in bin retrieval time with the proposed policy}\label{tab:thr_improve}
    \centering
    \small
        \begin{tabular}{@{}p{0.1\textwidth}p{0.15\textwidth}p{0.08\textwidth}p{0.08\textwidth}p{0.08\textwidth}p{0.08\textwidth}p{0.08\textwidth}p{0.08\textwidth}p{0.08\textwidth}@{}}
            \toprule
            Initial & Reduction (\%) & \multicolumn{7}{c}{Threshold (s)}\\
            BGC/ & Compare to & $30$& $40$&$50$& $60$& $70$& $80$& $90$\\
            \midrule
            0  & Delayed   &$77.64$    &$91.68$	&$96.12$  &$98.46$	&$99.51$	&$99.38$	&$99.13$ \\
            0   & Immediate &$95.68$	&$95.79$	&$97.21$	&$95.27$	&$95.83$	&$83.33$	&$60.00$ \\
            \cmidrule(r){1-2}   
            40  & Delayed  &$65.97$	&$69.39$	&$71.98$	&$74.30$	&$74.35$	&$76.82$	&$79.13$ \\
            40  & Immediate &$89.41$	&$85.47$	&$85.56$	&$79.18$	&$68.68$	&$68.58$	&$70.07$ \\
            \cmidrule(r){1-2}
            100 & Delayed &$56.04$	&$52.60$	&$50.95$	&$49.10$	&$48.12$	&$49.47$	&$56.19$ \\
            100 & Immediate &$81.63$	&$70.69$	&$68.18$	&$57.37$	&$47.99$	&$50.83$	&$60.26$ \\
            \bottomrule
        \end{tabular}
\end{table*}

Next, we look at the distributions of longer bin retrieval times ($\geq30$s), resulting in a longer waiting time at the workstation.
In \Cref{tab:thr_improve}, the proposed policy reduces the number of bin requests with retrieval time greater than $30s$ (and longer) by at least more than $50\%$ compared to the other two policies. 
This result indicates that the proposed policy reduces the frequency of bin retrieval time exceeding certain time thresholds, thus improving service quality.

\begin{figure*}[ht]
    \centering
    \begin{subfigure}{0.49\textwidth}
        \centering
        \includegraphics[width=\textwidth]{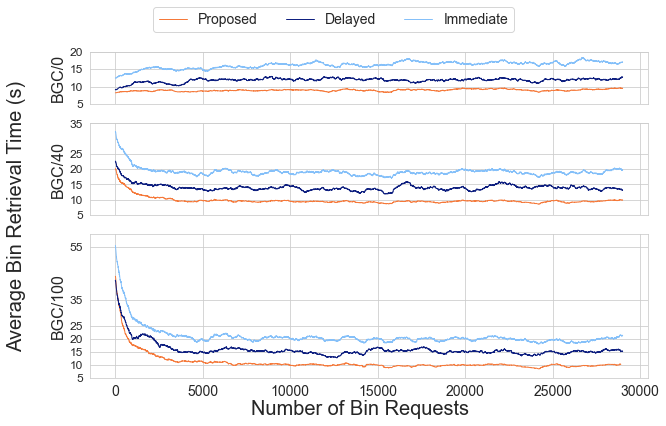}
        \caption{Moving average}
        \label{fig:averageBRT}
    \end{subfigure}
    \begin{subfigure}{0.49\textwidth}
        \centering
        \includegraphics[width=\textwidth]{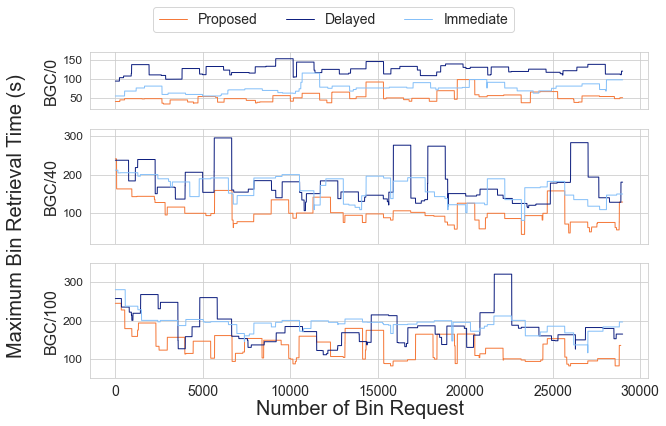}
        \caption{Moving max}
        \label{fig:maxBRT}
    \end{subfigure}
    \caption{Rolling plots with the smoothing interval of 1000 bin requests}
    \label{fig:movingBRT}
\end{figure*}

\Cref{fig:averageBRT} presents the moving average bin retrieval time with the smoothing interval of 1000 bin requests in the 100-hour simulation for the three policies and three scenarios.
In the BGC/0 scenario, starting with the optimal BGC, we see a small increase for three policies. The optimal BGC is disrupted by bin requests where the target bin is not in the surface layer, and this results in this small increase. The increase in the line representing the proposed policy is very small as the BGC is transformed among the set of the equivalent optimal BGCs.
In BGC/40 and BGC/100 scenarios, starting with a non-optimal BGC that has some popular bins stored near the bottom and some less popular bins near the top of the grid, the initial higher average bin retrieval cost is due to retrieving the deeply buried popular bins. As popular bins move to the top and unpopular bins sink to the bottom while processing bin requests, the average bin retrieval time decreases. 
With the proposed policy, the BGC transforms into a quasi-equivalent optimal BGC and/or an equivalent optimal BGC, resulting in a stable average time thereafter.

In scenarios BGC/40 and BGC/100 with the proposed policy, our simulation shows that these BGCs are transformed into quasi-equivalent optimal BGCs (with $\varepsilon=20\%$) for the first time after completing approximately $8000$ bin requests, and then are transformed among quasi-equivalent optimal BGCs.
This value explains well the lines of the proposed policy shown in \Cref{fig:averageBRT}, which become stable after the $8000$th bin request.
However, these BGCs are not transformed into equivalent optimal BGCs until the end of the simulation, as there are some out-of-place bins that have zero popularity (see \Cref{fig:binP}).
We run simulations with additional scenarios BGC/10 and BGC/20, where we set all out-of-place bins to have non-zero popularity (with the bin ID from $\llbracket 800\rrbracket$). 
The initial BGC/10 (resp.\ BGC/20) is transformed into an equivalent optimal BGC for the first time after completing about $3100$ (resp.\ $6600$) bin requests and then is transformed among the set of equivalent optimal BGCs.

An intriguing observation is that the line representing the proposed policy consistently remains below the lines representing the delayed policy (approximately $30\%$ improvement on average) and the immediate policy (approximately $50\%$ improvement on average) throughout the entire plot for any scenario.
We also observe that the lines representing the proposed policy stabilized at the same value for the three different initial BGCs. 
The lines (stabilized part) representing the other two policies increase as the randomization from the optimal BGC increases.
The moving maximum bin retrieval time with the smoothing interval of 1000 bin requests presented in \Cref{fig:maxBRT} indicates that the proposed policy effectively reduces the maximum bin retrieval time for the three scenarios.

In summary, the proposed policy effectively reduces the bin retrieval time.
More specifically, the proposed policy improves the performance of the system compared to the delayed policy and the immediate policy in the following aspects:
\begin{enumerate}
    \item More of the bin retrieval times are clustered in ranges of lower values, and the range (IQR) remains the same for all scenarios.
    \item The counts of bin retrieval times that exceed given thresholds (e.g., $\geq 30s$) are significantly reduced.
    \item The average bin retrieval times for any small portion within the simulation are always the lowest and stabilize around the same value for different scenarios.
    \item The maximum bin retrieval times for any small portion within the simulation are below the lines representing the other two policies for most of the time.
\end{enumerate}
The proposed policy guarantees the transformation of any initial BGC into a quasi-equivalent optimal BGC.
The proposed policy guarantees the transformation of initial BGCs with all out-of-place bins having bin popularity greater than zero into an equivalent optimal BGC.
\subsection{Robot Working Time}\label{Robot Runtime}

In this section, we evaluate the performance of the proposed methods measured via three metrics: the time of the delivery motion (i.e., horizontal movements of robots on top of the grid), the time of the gripper motion (i.e., vertical movements of the gripper of robots), and the overall time that is the sum of delivery and gripper times.
We compare the simulation result of the proposed policy with that of the delayed policy and the immediate policy, and the results are summarized in \Cref{tab:Robot Runtime Reduction}. 
Here, values for the three metrics are expressed as percent improvement relative to the other two state-of-the-art baselines.

\begin{table*}[ht]
\caption{Improvement in robot working time with the proposed policy}
\label{tab:Robot Runtime Reduction}
\centering
\small
\begin{tabular}{@{}lllllll@{}}
\toprule
Initial  & \multicolumn{3}{c}{Proposed vs.\ Delayed (\%)} & \multicolumn{3}{c}{Proposed vs.\ Immediate (\%)} \\
BGC & Overall & Delivery Motion & Gripper Motion & Overall & Delivery Motion & Gripper Motion\\
\cmidrule(r){1-1} \cmidrule(r){2-4} \cmidrule{5-7} 
0       & $42.99$   & $57.61$   & $35.91$   & $41.70$    & $50.51$     & $38.17$      \\
40      & $35.58$   & $55.07$   & $25.21$   & $34.43$    & $48.76$     & $28.00$      \\
100     & $24.06$   & $48.16$   & $10.22$   & $20.67$    & $39.49$     & $11.55$      \\
\bottomrule
\end{tabular}
\end{table*}

We observe that all percentages are positive, which supports the improvement of the proposed policy over the other two policies in simulations with different degrees of randomness of the initial BGC.
In the scenario starting with an optimal BGC (BGC/0), the proposed policy keeps the BGC transformed among the set of equivalent optimal BGCs, and we observed a maximum improvement of approximately $42\%$.  
Although the improvement decreases as the randomness of the initial BGC increases, the proposed policy continues to reduce the overall working time by almost $20\%$, which is the minimum improvement in the scenario starting with BGC/100. 

In summary, the proposed policy effectively reduces the robot working time compared to the delayed policy and the immediate policy for finishing the same amount of bin requests with the same initial BGC, contributing to energy efficiency for system operation.
\section{Conclusions}
\label{lab: Sec conclusion}
In this paper, we studied the problem of minimizing the bin retrieval time for a robotic-based compact storage and retrieval system.
To achieve the objective, we focused on reducing the digging depth of the requested bins in the grid.
We derived the optimal BGC that ensures the minimum expected bin retrieval cost for a single bin request.
However, the optimal BGC cannot be maintained when the requested bin is not in the surface layer.
Therefore, we constructed a set of equivalent optimal BGCs and a set of quasi-equivalent optimal BGCs, and proposed a layer complete policy that transforms any BGC into an equivalent optimal BGC (unguaranteed) and a quasi-equivalent optimal BGC (guaranteed) while processing a series of bin requests.
The set of equivalent optimal BGCs and the set of quasi-equivalent optimal BGCs are each positively invariant under the proposed policy, which ensures that the proposed structures are maintainable while processing a series of bin requests after being achieved.
The proposed solution naturally transforms any BGC into a target structure (an equivalent optimal BGC or a quasi-equivalent optimal BGC) without affecting the normal bin requests, which improves the performance of the system and makes it responsive to changes in bin demand.

We compared the proposed policy with the delayed policy and the immediate policy in discrete event simulation models. 
Our results showed that under the layer complete policy, more than $50\%$ of the required bins are retrieved from the surface layer, effectively reducing the bin retrieval time (approximately $30\%$ less than the delayed policy and $50\%$ less than the immediate policy) and the number of bin requests that exceed certain time thresholds by at least $48\%$.
In addition, the proposed policy outperforms the other two policies in saving robot working time (at least $20\%$ improvement) by completing the same amount of bin requests.

In future work, we will include the impact of the number and location of workstations, which is a significant factor that affects the delivery motion. 
In addition, a more accurate delivery motion with intelligent route planning will improve the reliability of the modeling.
Last but not least, an interesting and challenging direction waiting to be discovered is to retrieve multiple requested bins in one digging.
To solve this problem, we can refer to extensive work that has been done in other AS/RS, such as closely
store the products that are usually ordered together (e.g., in the same stack or in
adjacent stacks for RCS/RS), group and sequence bin requests before they enter the system, and assign delivery and digging tasks to robots intelligently. 

\ACKNOWLEDGMENT{%
This work is supported in part by the Natural Sciences and Engineering Research Council of Canada (NSERC).
The authors would like to extend their heartfelt appreciation to FlexSim for generously providing
the educational license for the use of their software. Their support has been invaluable in
facilitating the research and analysis conducted for this study.
}




\bibliographystyle{informs2014trsc}
\bibliography{Bibliography.bib}


%

%
%
\clearpage
\begin{APPENDICES}
\section{Notations for System Analysis}
\label{APX:Notations for System Analysis}
\begin{table}[ht]
\small
    \caption{Notations for System Analysis}
    \centering
    \begin{tabularx}{\columnwidth}{@{}p{0.1\linewidth} p{0.84\linewidth}l@{}}
    \toprule
    Notation     & Description  \\
    \midrule
    $\mathbf{B}$    & A feasible BGC, which is an $H \times M$ matrix. \\
    $\mathbb{B}$    & Set of all optimal BGCs. \\ 
    $\hat{\mathbb{B}}$   & Set of equivalent optimal BGCs, $\hat{\mathbb{B}}=\cup_{\mathbf{B}\in\mathbb{B}}[\mathbf{B}]$.\\
    $\mathbf{B}^k$    & BGC after the $k$th bin request, $k\in\{0,1,2,\dots\}$, where $\mathbf{B}^{0}$ represents the initial BGC.\\
    $[\mathbf{B}]$  & Equivalent class of $\mathbf{B}$.\\
    $[\mathbf{B}]_{\varepsilon}$  & Quasi-equivalent class of $\mathbf{B}$ with respect to $\varepsilon$.\\
    $\mathbf{B}_{l,m}$        & The ID of the bin in the cell located in stack $m$ and layer $l$. \\
    $H$     & Grid height, by the number of cells. Also, the maximum fill level of a stack.   \\
    $\overline{H}$   & The minimum fill level of a system, $\overline{H}=H(1-\tau)$. \\
    $\textbf{\emph{h}}$   & Fill level vector of the grid. \\
    $h_c$   & Fill level of occupied stacks within the grid. \\
    $h_e$   & Empty level of occupied stacks within the grid, $h_e=H-h_c$. \\
    $l$     & Layer in the grid, $l\in\llbracket H\rrbracket$.   \\
    $M$     & Number of the stacks in the grid.\\
    $\llbracket M\rrbracket,\ m$   & Set includes all stack IDs, $m\in\llbracket M\rrbracket$.\\
    $m_f$   & Number of occupied stacks, $m_f\leq M$. \\
    $N$     & Number of the bins stored in the grid, $N\leq H\cdot M$.\\
    $\llbracket N\rrbracket,\ n$   & Set includes all bin IDs, $n\in\llbracket N\rrbracket$.\\
    $\Omega(\mathbf{B})$    & Function computes the fill level of a BGC $\mathbf{B}$, $\Omega(\mathbf{B})=h_c$.\\
    $p_n$   & Popularity of bin $n$, $\sum_{n\in \llbracket N\rrbracket}p_n=1$.    \\
    $p_{\mathbf{B}_{l,m}}$     & Popularity of the bin stored in stack $m$ and layer $l$ within the grid, $\sum_{l\in \llbracket H\rrbracket}\sum_{m\in \llbracket M\rrbracket}p_{\mathbf{B}_{l,m}}=1$. \\
    $\pi_l(\mathbf{B})$   & Function computes the probability that the bin is in layer $l$ within a BGC $\mathbf{B}$, $\pi_l(\mathbf{B})=\sum_{m\in \llbracket M\rrbracket} p_{\mathbf{B}_{l,m}}$, $\sum_{l\in \llbracket H\rrbracket}\pi_l(\mathbf{B})=1$.    \\
    $\Phi$  & Set of all feasible BGCs.\\
    $\boldsymbol{\sigma}$ & Sequence of bin requests, where $\sigma_k\in\boldsymbol{\sigma}$ represents the bin ID of the $k$th requested bin.\\
    $\tau$  & Percentage of the total cells that were put aside for expansion in the future.\\
    $\varepsilon$ & Percentage of popular bins out of all bins.\\
    \bottomrule
    \end{tabularx}
    \label{tab:notation}
\end{table}

\section{Lookup table to compute the cost of placing the bins removed above the target bin on top of the nearby stacks}
\label{APX:LUT}
\begin{table*}[ht]
    \caption{LUT of $C_{r,2}(l,h_e)$, where symbol - means the value does not exist.}
    \label{tab:Cr2_LUT}
    \centering
    \small
        \begin{tabularx}{\textwidth}{XXXXXXXXXXXXXXXXXXXXXXX}
            \toprule
            &\multicolumn{22}{c}{$l$}\\
            $h_e$ & 1 & 2 & 3 & 4 & 5 & 6 & 7 & 8 & 9 & 10 & 11 & 12 & 13 & 14 & 15 & 16 & 17 & 17 & 19 & 20 & 21 & 22 \\
            \cmidrule(r){1-1} \cmidrule(l){2-23}
                0 & 0 & 0 & 0 & 0 & 0 & 0 & 0 & 0 & 0 & 0 & 0 & 0 & 0 & 0 & 0 & 0 & 0 & 0 & 0 & 0 & 0 & 0 \\
                1 & - & 0 & 2 & 2 & 4 & 4 & 6 & 6 & 8 & 8 & 10 & 10 & 12 & 12 & 14 & 14 & 16 & 16 & 18 & 18 & 20 & 20 \\
                2 & - & - & 0 & 4 & 6 & 6 & 10 & 12 & 12 & 16 & 18 & 18 & 22 & 24 & 24 & 28 & 30 & 30 & 34 & 36 & 36 & 40 \\
                3 & - & - & - & 0 & 6 & 10 & 12 & 12 & 18 & 22 & 24 & 24 & 30 & 34 & 36 & 36 & 42 & 46 & 48 & 48 & 54 & 58 \\
                4 & - & - & - & - & 0 & 8 & 14 & 18 & 20 & 20 & 28 & 34 & 38 & 40 & 40 & 48 & 54 & 58 & 60 & 60 & 68 & 74 \\
                5 & - & - & - & - & - & 0 & 10 & 18 & 24 & 28 & 30 & 30 & 40 & 48 & 54 & 58 & 60 & 60 & 70 & 78 & 84 & 88 \\
                6 & - & - & - & - & - & - & 0 & 12 & 22 & 30 & 36 & 40 & 42 & 42 & 54 & 64 & 72 & 78 & 82 & 84 & 84 & 96 \\
                7 & - & - & - & - & - & - & - & 0 & 14 & 26 & 36 & 44 & 50 & 54 & 56 & 56 & 70 & 82 & 92 & 100 & 106 & 110 \\
                8 & - & - & - & - & - & - & - & - & 0 & 16 & 30 & 42 & 52 & 60 & 66 & 70 & 72 & 72 & 88 & 102 & 114 & 124 \\
                9 & - & - & - & - & - & - & - & - & - & 0 & 18 & 34 & 48 & 60 & 70 & 78 & 84 & 88 & 90 & 90 & 108 & 124 \\
                10 & - & - & - & - & - & - & - & - & - & - & 0 & 20 & 38 & 54 & 68 & 80 & 90 & 98 & 104 & 108 & 110 & 110 \\
                11 & - & - & - & - & - & - & - & - & - & - & - & 0 & 22 & 42 & 60 & 76 & 90 & 102 & 112 & 120 & 126 & 130\\
            \bottomrule
        \end{tabularx}
\end{table*}
\end{APPENDICES}

\end{document}